%% file: arXiv_BanachDA_Learning.tex
\documentclass[12pt]{colt2016_plain} 

\title[Minimizing Regret on Reflexive Banach Spaces]{Minimizing Regret on Reflexive Banach Spaces and\\ Learning Nash Equilibria in Continuous Zero-Sum Games}
\usepackage{times}

\usepackage{hyperref}
\usepackage{url}
\usepackage{amsfonts}
\usepackage{graphicx}
\usepackage{wrapfig}
\usepackage{enumerate}
\usepackage{enumitem}
\usepackage{algorithm}
\usepackage[noend]{algorithmic}
\usepackage{mathtools}
\usepackage{caption}
\usepackage{multirow}

\input{latex_commands}

\newcommand\column[1]{\begin{tabular}{c} #1 \end{tabular}}

\newtheorem{assumption}{Assumption}

 \coltauthor{%
 \Name{Maximilian Balandat} \Email{balandat@eecs.berkeley.edu}\\
 \Name{Walid Krichene} \Email{walid@eecs.berkeley.edu}\\
 \Name{Claire Tomlin} \Email{tomlin@eecs.berkeley.edu}\\
 \Name{Alexandre Bayen} \Email{bayen@berkeley.edu}\\
 \addr Department of Electrical Engineering and Computer Sciences\\
 University of California, Berkeley. 
 }

\begin{document}
\maketitle

\begin{center}
\vspace*{-8ex}
Current version: June 3, 2016
\vspace*{2ex}
\end{center}

\begin{abstract}
We study a general version of the adversarial online learning problem. We are given a decision set~$\Xcal$ in a reflexive Banach space~$X$ and a sequence of reward vectors in the dual space of~$X$. At each iteration, we choose an action from~$\Xcal$, based on the observed sequence of previous rewards. Our goal is to minimize regret, defined as the gap between the realized reward and the reward of the best fixed action in hindsight. Using results from infinite dimensional convex analysis, we generalize the method of Dual Averaging (or Follow the Regularized Leader) to our setting and obtain general upper bounds on the worst-case regret that subsume a wide range of results from the literature. Under the assumption of uniformly continuous rewards, we obtain explicit anytime regret bounds in a setting where the decision set is the set of probability distributions on a compact metric space~$S$ whose Radon-Nikodym derivatives are elements of $L^p(S)$ for some $p>1$. Importantly, we make no convexity assumptions on either the set $S$ or the reward functions. We also prove a general lower bound on the worst-case regret for any online algorithm. We then apply these results to the problem of learning in repeated  continuous two-player zero-sum games, in which players' strategy sets are compact metric spaces. In doing so, we first prove that if both players play a Hannan-consistent strategy, then with probability~1 the empirical distributions of play weakly converge to the set of Nash equilibria of the game. We then show that, under mild assumptions, Dual Averaging on the (infinite-dimensional) space of probability distributions indeed achieves Hannan-consistency. Finally, we illustrate our results through numerical examples.  
\end{abstract}



\section{Introduction}
\label{sec:Intro}

Regret analysis is a general technique for designing and analyzing algorithms for sequential decision problems in adversarial or stochastic settings~\citep{audibert2009minimax,bubeck2012regret}. Online learning algorithms have many applications, including in machine learning~\citep{xiao2010dual}, portfolio optimization~\citep{cover91universal}, and online convex optimization~\citep{Hazan:2007aa}. 
A particularly interesting role that regret plays manifests in the study of repeated play of finite games~\citep{Hart:2001aa}. It is well known, for example, that in a two-player zero-sum finite game, if both players play according to a Hannan-consistent strategy~\citep{Hannan:1957aa}, their (marginal) empirical distributions of play almost surely converge to the set of Nash equilibria of the game~\citep{Cesa-Bianchi:2006aa}. Moreover, it can be shown that playing a strategy that achieves sublinear regret almost surely guarantees Hannan-consistency.

A natural question to ask is whether a similar result holds for games in which the action sets are infinite. In this paper we show that this is in fact true. In particular, we prove that in a continuous two-player zero sum game over compact (not necessarily convex) metric spaces, if both players follow a Hannan-consistent strategy, then with probability~$1$, their empirical distributions of play weakly converge to the set of Nash equilibria of the game. 
This in turn raises another important question: Do algorithms that ensure Hannan-consistency exist for such games? More generally, can one develop algorithms that guarantee sub-linear growth of the worst-case regret? We answer these questions affirmatively as well in this article. To this end, we develop a general framework to study the Dual Averaging (or Follow the Regularized Leader) method on reflexive Banach spaces. This framework subsumes a wide range of existing results in the literature, including algorithms for online learning on finite sets (e.g. the Hedge Algorithm~\citep{Arora:2012aa}), and finite-dimensional online convex optimization (e.g. Exponentially Weighted Online Optimization~\citep{Hazan:2007aa}). Our results are related to~\citep{Lehrer:2003aa} and~\citep{Sridharan:2010aa}. \cite{Lehrer:2003aa} conducts an abstract analysis, providing necessary geometric conditions for Blackwell approachability in infinite-dimensional spaces, but no implementable algorithm that guarantees Hannan-consistency. \cite{Sridharan:2010aa} provide general regret bounds for Mirror Descent (MD) under the assumption that the strategy set is uniformly bounded in the norm of the Banach space. We make no such assumption. In fact, for our applications 
in Section~\ref{sec:OOCUC} this is typically not the case.

Given a convex subset~$\Xcal$ of a reflexive Banach space~$X$, the generalized Dual Averaging (DA) method maximizes, at each iteration, the cumulative past rewards (which are elements of~$X^*$, the dual space of~$X$) plus a regularization term $h$. We show that under certain conditions, the maximizer in the DA update is the Fr\'echet gradient $Dh^*$ of the regularizer's conjugate function. In doing so, we develop a novel characterization of the duality between essential strong convexity of~$h$ and essential Fr\'echet differentiability of~$h^*$ in reflexive Banach spaces, which may be of independent interest.

We apply these general results to the problem of minimizing regret when the rewards are uniformly continuous functions over a compact metric space~$S$. Importantly, we do not assume convexity of either $S$ or the reward functions. Under the assumption that $S$ admits a locally $Q$-regular Borel measure~$\mu$, we lift this non-convex finite dimensional problem on $S$, to the convex infinite-dimensional problem on the set of probability distributions on $S$ that are absolutely continuous with respect to $\mu$, and whose Radon-Nikdoym derivatives are in $X = L^p(S,\mu)$ for some $p>1$. We provide explicit bounds for a class of regularizers, which guarantee sublinear growth of the worst-case regret. We also prove a general lower bound on the regret for any online algorithm.

The paper is organized as follows: In Section~\ref{sec:RegMin} we introduce and provide a general analysis of Dual Averaging in reflexive Banach spaces. In Section~\ref{sec:OOCUC} we apply these results to obtain explicit regret bounds on compact metric spaces with uniformly continuous reward functions. We use these results in Section~\ref{sec:LearnCont} in the context of learning Nash equilibria in continuous two-player zero sum games, for which we provide numerical examples in Section~\ref{sec:Examples}. 
All proofs are given in Appendix~\ref{appsec:OmittedProofs}.

\section{Regret Minimization on Reflexive Banach Spaces}
\label{sec:RegMin}

Consider a sequential decision problem in which we are to choose a sequence~$(x_1,x_2,\dotsc)$ of actions from some feasible subset~$\Xcal$ of a reflexive Banach space~$X$, and seek to maximize a sequence $(u_1(x_1), u_2(x_2), \dotsc)$ of \emph{rewards}, where the $u_\tau:  X \rightarrow \Rbb$ are elements of a given subset~$\Ucal \subset X^*$, with $X^*$  the dual space of~$X$. 
We assume that~$x_t$, the action chosen at time $t$, may only depend on the sequence of previously observed reward vectors $(u_1, \dotsc, u_{t -1})$. We call any such algorithm an  \emph{online algorithm}. 
We consider the \emph{adversarial} setting, i.e.,  we do not make any distributional assumptions on the rewards. 
In particular, they could be picked maliciously by some adversary. 

The notion of \emph{regret} is a standard measure of performance for such a sequential decision problem. For a sequence~$(u_1,\dotsc,u_t)$ of reward vectors, and a sequence of decisions~$(x_1,\dotsc,x_t)$ produced by an algorithm, the regret of the algorithm with respect to a (fixed) decision~$x\in \Xcal$ is the gap between the realized reward and the reward under~$x$. In other words,\\[-3ex]
\begin{align}
R_t(x) := \sum_{\tau=1}^t u_\tau(x) - \sum_{\tau=1}^t u_\tau(x_\tau)
\end{align}
%
%
%
The \emph{worst-case regret} is defined as \\[-4ex]
\begin{align}
\Rcal_t := \sup_{x \in \Xcal} R_t(x)
\end{align}
An algorithm is said to have \emph{sublinear regret} if for any sequence $(u_t)_{t\geq 1}$ in the set of admissible reward functions $\Ucal$, the worst-case regret grows sublinearly, i.e. $\limsup_{t} \Rcal_t/t \leq 0$.
\begin{example}[Finite Action Sets]
\label{ex:DAContTime:RegMinFiniteAction}
Consider a finite action set $S = \{1, \dotsc, n\}$,  let $X=X^*=\Rbb^n$, and let $\Xcal = \Delta_{n-1}$, the probability simplex in $\Rbb^n$. In this case, a reward function on $S$ is simply a vector $u\in \Rbb^n$, such that the $i$-th element $u_i$ is the reward of action~$i$. A choice $x\in \Xcal$ corresponds to a randomization over the $n$ discrete actions in $S$. This is the classic setting of many regret-minimizing algorithms in the literature.
\end{example}

\begin{example}[Online Optimization on Compact Metric Spaces]
\label{ex:DAContTime:OOptCompact}
Let $S$ be a compact metric space, and let~$\mu$  be a finite measure on~$S$. Consider the Hilbert Space $X=X^*= L^2(S,\mu)$ and let $\Xcal = \{x\in X \;\suchthat\; x \geq 0 \;\text{a.e.}, \|x\|_1 = 1\}$. The set~$S$ is again the set of feasible actions. A reward function is an $L^2$-integrable function on $S$, and each choice $x\in \Xcal$ corresponds to a probability distribution (absolutely continuous w.r.t. to $\mu$) over the set of actions. We will further explore a more general variant of this problem in Section~\ref{sec:OOCUC}. 
\end{example}

In this Section, we prove a general bound on the worst-case regret for the Dual Averaging method. Dual Averaging was introduced by~\cite{nesterov2009primaldual} for (finite dimensional) convex optimization, and was since applied to the online learning setting, for example by~\cite{xiao2010dual}. In the finite dimensional case, the method works by solving, at each iteration, the maximization problem\\[-0.75ex]
\[
x_{t+1} = \argmax_{x \in \Xcal} \;\Bigl\langle \eta_t \textstyle\sum_{\tau = 1}^t u_\tau\,,\; x \Bigr\rangle - h(x)
\]
where $h$ is a strongly convex regularizer defined on $\Xcal\subset \Rbb^n$ and $(\eta_t)_{t\geq0}$ is a sequence of learning rates. The regret analysis of the method relies on the duality between strong convexity and smoothness for a convex function and its conjugate \cite[Lemma 1]{nesterov2009primaldual}, see also~\cite[Theorem 26.3]{rockafellar1997convex}. To generalize the Dual Averaging method to our Banach space setting, we first need an analogous duality result. We develop such a result in Theorem~\ref{thm:RegMin:Prelims:MainTheorem}. In particular, we show that the correct notion of strong convexity in our setting is (uniform) essential strong convexity. Equipped with this duality result, we can analyze the regret of the Dual Averaging method and derive a general bound in Theorem~\ref{thm:RegMin:DiscTime:DTBounds}.

\subsection{Preliminaries}
\label{subsec:RegMin:Prelims}

Let $(X, \|\cdot\|)$ be a reflexive Banach space, and denote by $\langle\,\cdot\,,\,\cdot\,\rangle : X \times X^* \rightarrow \Rbb$ the canonical pairing between $X$ and its dual space $X^*$, so that $\langle x, \xi \rangle := \xi(x)$ for all $x\in X, \xi \in X^*$. By the \emph{effective domain} of an extended real-valued function  $f : X \rightarrow [-\infty, +\infty]$ we mean the set $\dom f = \{ x \in X \;\suchthat\; f(x) < +\infty \}$.  A function~$f$ is \emph{proper} if $f >-\infty$ and $\dom f$ is non-empty. 
The \emph{conjugate} or \emph{Legendre-Fenchel transform} of $f$ is the function $f^* : X^* \rightarrow [-\infty, +\infty]$ given by\\[-3.5ex]
\begin{align}
\label{eq:RegMin:Prelims:LegendreFenchel}
f^*(\xi) = \sup_{x\in X} \;\langle x, \xi \rangle - f(x)
\end{align}
for all $\xi \in X^*$. 
If $f$ is proper, lower semicontinuous and convex, its \emph{subdifferential} $\partial f$ is the set-valued mapping $\partial f(x) = \bigl\{ \xi \in X^* \;\suchthat\; f(y) \geq f(x) + \langle y-x, \xi \rangle \text{ for all } y\in X \bigr\}$. 
We define $\dom \partial f := \{x\in X \,\suchthat\, \partial f(x) \neq \emptyset \}$. 
Let $\Gamma$ denote the set of all convex, lower semicontinuous functions $\gamma: [0,\infty) \rightarrow [0,\infty]$ such that $\gamma(0)=0$, and  let
\begin{subequations}
\begin{align}
\Gamma_{\!U} &:= \bigr\{ \gamma \in \Gamma \;\suchthat\; \gamma(r) >0, \text{ for } r >0 \bigr\} \\
\Gamma_{\!L} &:= \bigr\{ \gamma \in \Gamma \;\suchthat\; \gamma(r)/r \rightarrow 0,  \text{ as }  r \rightarrow 0 \bigr\}
\end{align} 
\end{subequations}

We now introduce the appropriate definitions of strong convexity, differentiability and smoothness for our setting. Some related results are reviewed in Appendix~\ref{appsec:ConvAnalysis}.

\begin{definition}[Essential strong convexity~\citep{Stromberg:2011aa}]
\label{def:RegMin:Prelims:StrongConv}
\label{def:RegMin:Prelims:UnifStrongConv}
A proper convex lower semicontinuous function $f:X \rightarrow (-\infty, \infty]$ is \emph{essentially strongly convex} if 
\begin{enumerate}[label=(\roman*), itemsep=0ex]
\item $f$ is strictly convex on every convex subset of $\dom \partial f$
\label{def:RegMin:Prelims:StrongConv:i}
\item $(\partial f)^{-1}$ is locally bounded on its domain
\label{def:RegMin:Prelims:StrongConv:ii}
\item for every $x_0 \in \dom \partial f$ there exists $\xi_0 \in X^*$ and $\gamma \in \Gamma_{\!U}$ such that\\[-3.25ex]
\label{def:RegMin:Prelims:StrongConv:iii}
\begin{align}
\label{eq:RegMin:Prelims:StrongConv:iii}
f(x) \geq f(x_0) + \langle x-x_0, \xi_0 \rangle + \gamma(\|x-x_0\|), \quad \forall x\in X
\end{align}
\end{enumerate}
If~\eqref{eq:RegMin:Prelims:StrongConv:iii} holds with $\gamma$ independent of $x_0$, then $f$ is said to be \emph{uniformly essentially strongly convex} with modulus $\gamma$.
\end{definition}

\begin{definition}[Essential Fr\'echet differentiability~\citep{Stromberg:2011aa}]
\label{def:RegMin:Prelims:Frechet}
A proper convex lower semicontinuous function $f:X \rightarrow (-\infty, \infty]$ is \emph{essentially Fr\'echet differentiable} if $\interior \dom f \neq \emptyset$, $f$ is Fr\'echet differentiable on $\interior \dom f$ with Fr\'echet derivative $D$, and $\| D f(x_j)\|_* \rightarrow \infty$ for any sequence $(x_j)_j$ in $\interior \dom f$ converging to some boundary point of $\dom f$. 
\end{definition}

\begin{definition}[Essential strong smoothness]
\label{def:RegMin:Prelims:EssSmooth}
\label{def:RegMin:Prelims:UnifSmooth}
A proper Fr\'echet differentiable function $f:X \rightarrow (-\infty, \infty]$ is essentially strongly smooth if $\,\forall\,x_0 \in \dom\partial f,\;\exists\, \xi_0\in X^*,\; \kappa \in \Gamma_{\!L}$ such that \\[-2.75ex]
\begin{align}
\label{eq:RegMin:Prelims:EssSmooth}
f(x) \leq f(x_0) + \langle 
\xi_0
, x - x_0 \rangle + \kappa(\| x - x_0 \|),\quad \forall\,x \in X
\end{align}
If \eqref{eq:RegMin:Prelims:EssSmooth} holds with $\kappa$ independent of $x_0$, then $f$ is said to be \emph{uniformly essentially strongly smooth} with modulus $\kappa$.
\end{definition}

We are now ready to give our main duality result:
\begin{theorem}
\label{thm:RegMin:Prelims:MainTheorem}
Let $f: X \rightarrow (-\infty, +\infty]$ be proper, lower semicontinuous and uniformly essentially strongly convex with modulus~$\gamma \in \Gamma_{\!U}$. Then 
\begin{enumerate}[label=(\roman*)]
\item $f^*$ is proper and essentially Fr\'echet differentiable with Fr\'echet derivative \\[-3ex]
\label{thm:RegMin:Prelims:MainTheorem:FrechetGradient}
\begin{align}
\label{eq:RegMin:Prelims:MainTheorem:FrechetGradient}
D f^*(\xi) = \argmax_{x\in X} \;\langle x, \xi \rangle - f(x)
\end{align}
If, in addition, $\tilde\gamma(r) := \gamma(r)/r$ is strictly increasing, then  \\[-3ex]
\begin{align}
\|Df^*(\xi_1) - Df^*(\xi_2)\| \leq \tilde\gamma^{-1} \bigl( \|\xi_1 - \xi_2\|_* / 2 \bigr)
\label{eq:RegMin:Prelims:MainTheorem:ModContGrad}
\end{align}
In other words, $Df^*$ is uniformly continuous with modulus of continuity $\chi(r) = \tilde\gamma^{-1}(r/2)$.
\item $f^*$ is uniformly essentially smooth with modulus $\gamma^*$.
\label{thm:RegMin:Prelims:MainTheorem:StrongSmooth}
\end{enumerate}
\end{theorem}

\begin{corollary}
\label{cor:RegMin:Prelims:MainTheorem:Superlinearity}
If $\gamma(r) \geq C \,r^{1+\kappa},\;\forall\,r\geq 0$ then $\|Df^*(\xi_1) - Df^*(\xi_2)\| \leq (2C)^{-1/\kappa} \|\xi_1 - \xi_2\|_*^{1/\kappa}$. In particular, with $\gamma(r) = \frac{K}{2}r^2$, Definition~\ref{def:RegMin:Prelims:StrongConv} becomes the classic definition of $K$-strong convexity, and~\eqref{eq:RegMin:Prelims:MainTheorem:ModContGrad} yields the result familiar from the finite-dimensional case that the gradient $Df^*$ is $1/K$ Lipschitz with respect to the dual norm~\cite[Lemma 1]{nesterov2009primaldual}.
\end{corollary}

\subsection{Dual Averaging in Reflexive Banach Spaces}
\label{subsec:DAContTime}

%
%
%


We call a proper convex function $h: X \rightarrow (-\infty, +\infty]$ a \emph{regularizer function} on a set $\Xcal \subset X$ if $h$ is essentially strongly convex and $\dom h = \Xcal$. We emphasize  that we do not assume~$h$ to be Fr\'echet-differentiable. 
Definition~\ref{def:RegMin:Prelims:StrongConv} in conjunction with Lemma~\ref{lem:RegMin:Prelims:PWEquiv} implies that for any regularizer function~$h$, the supremum of any function of the form $\langle \,\cdot\,, \xi \rangle - h(\,\cdot\,)$ over $X$, where $\xi\in X^*$, will be attained at a unique element of $\Xcal$, namely $Dh^*(\xi)$, the Fr\'echet gradient of $h^*$ at $\xi$. 

The Dual Averaging method with regularizer $h$ and a sequence of \emph{learning rates} $(\eta_t)_{t\geq1}$ generates a sequence of decisions using the simple update rule:\\[-1.5ex]
\[
x_{t+1} = Dh^*(\eta_t U_t)
\]
where $U_t = \sum_{\tau = 1}^t u_\tau$ and $U_0 := 0$. The following theorem provides a general regret bound.

\begin{theorem}[Dual Averaging Regret]
\label{thm:RegMin:DiscTime:DTBounds}
Let $h$ be a uniformly essentially strongly convex regularizer on~$\Xcal$ with modulus~$\gamma$. Let $(\eta_t)_{t\geq 1}$ be a positive non-increasing sequence of learning rates. Then, for any sequence of payoff functions $(u_t)_{t\geq 1}$ in $X^*$, the sequence of plays $(x_t)_{t\geq 0}$ given by \\[-2.75ex]
\begin{align}
\label{eq:RegMin:DiscTime:DTBounds:Strategy}
x_{t+1} = Dh^*\bigl( \eta_t \textstyle \sum_{\tau=1}^t u_\tau \bigr) \\[-4.25ex] \nonumber
\end{align}
ensures that\\[-4ex]
\begin{align}
\label{eq:RegMin:DiscTime:DTBounds:Bound}
R_t(x) := \sum_{\tau=1}^t \langle u_\tau, x \rangle - \sum_{\tau=1}^t \langle u_\tau, x_\tau \rangle \leq \frac{h(x)-\underline{h}}{\eta_t} +  \sum_{\tau=1}^t \|u_\tau\|_* \, \tilde\gamma^{-1} \Bigl( \frac{\eta_{\tau-1} }{2}\|u_\tau\|_* \Bigr)
\end{align}
where $\underbar h = \inf_{x \in \Xcal} h(x)$, $\tilde\gamma(r) := \gamma(r) / r$ and $\eta_0 := \eta_1$. 
\end{theorem}

Note that it is possible to obtain a regret bound similar to~\eqref{eq:RegMin:DiscTime:DTBounds:Bound} also in a continuous-time setting. In fact, following~\cite{Kwon:2014aa}, we derive the bound~\eqref{eq:RegMin:DiscTime:DTBounds:Bound} by first proving a bound on a suitably defined notion of continuous-time regret, and then bounding the difference between the continuous-time and discrete-time regrets. This analysis is detailed in Appendix~\ref{appsec:DAContTime}.

Theorem~\ref{thm:RegMin:DiscTime:DTBounds} provides a bound on the regret $R_t(x)$ with respect to a particular choice $x\in \Xcal$. Recall that the worst-case regret is defined as $\Rcal_t := \sup_{x\in \Xcal} R_t(x)$. In the finite-dimensional seting of Example~\ref{ex:DAContTime:RegMinFiniteAction} the set $\Xcal$ is compact, so any continuous regularizer~$h$ will be bounded, and hence taking the supremum over $x$ in~\eqref{eq:RegMin:DiscTime:DTBounds:Bound} poses no issue. However, this is not the case in our general setting, as the regularizer may be unbounded on $\Xcal$. For instance, consider Example~\ref{ex:DAContTime:OOptCompact} with the entropy regularizer $h(x) = \int_{S} x(s) \log(x(s)) ds$, which is easily seen to be unbounded on~$\Xcal$. As a consequence, obtaining a worst-case bound will in general require additional assumptions on the reward functions and the decision set~$\Xcal$. This will be investigated in detail in Section~\ref{sec:OOCUC}.

\begin{corollary}
\label{cor:RegMin:DiscTime:DTBounds:superlinear}
Suppose that $\gamma(r) \geq C \,r^{1+\kappa}, \;\forall \,r\geq 0\,$ for some $ C>0$ and $\kappa >0$. Then \\[-3.25ex]
\begin{align}
\label{eq:RegMin:DiscTime:DTBounds:superlinear:etageneral}
R_t(x) \leq \frac{h(x)-\underline{h}}{\eta_t} + (2C)^{-1/\kappa}  \sum_{\tau=1}^t \eta_{\tau-1}^{1/\kappa} \|u_\tau\|_*^{1+1/\kappa}
\end{align}
In particular, if $\|u_t\|_* \leq M$ for all $t$ and $\eta_t = \eta \, t^{-\beta}$, then \\[-2.5ex]
\begin{align}
\label{eq:RegMin:DiscTime:DTBounds:superlinear:etaspecific}
R_t(x) \leq \frac{h(x)-\underline{h}}{\eta}\, t^{\beta} + \frac{\kappa}{\kappa-\beta}  \Bigl(\frac{\eta}{2C}\Bigr)^{\!1/\kappa} M^{1+1/\kappa} \,t^{1-\beta/\kappa}
\end{align}
\end{corollary}

Assuming $h$ is bounded, optimizing over $\beta$ yields a rate of $R_t(x) = \Ocal(t^{\frac{\kappa}{1+\kappa}})$. In particular, if $\gamma(r) = \frac{K}{2}r^2$, which corresponds to the classic definition of strong convexity, then $R_t(x) = \Ocal(\sqrt{t})$. For non-vanishing $u_\tau$ we will need that $\eta_t\searrow 0$ for the sum in~\eqref{eq:RegMin:DiscTime:DTBounds:superlinear:etageneral} to converge. Thus we could get potentially tighter control over the rate of this term for $\kappa < 1$, at the expense of larger constants. 
\section{Online Optimization on Compact Metric Spaces with Uniformly Continuous Rewards$\!\!\!\!\!\!\!\!\!\!\!\!\!\!\!$}
\label{sec:OOCUC}

Motivated by Example~\ref{ex:DAContTime:OOptCompact}, in this section we apply our results to the problem of regret minimization on compact metric spaces under the additional assumption of uniformly continuous reward functions. Importantly, we make no assumptions on convexity of either the feasible set or the reward functions. Essentially, this can be seen as lifting the non-convex problem of minimizing a sequence of functions over the (possibly non-convex) set~$S$ to the convex (albeit infinite-dimensional) problem of minimizing a sequence of linear functionals over the convex subset~$\Xcal$ of probability measures over the vector space of measures on~$S$. 
This correspondance is illustrated in Figure~\ref{fig:OOCUC:lifting}.  
%
\begin{figure}[h]
\centering
\vspace{1ex}
\includegraphics[width=0.7\textwidth]{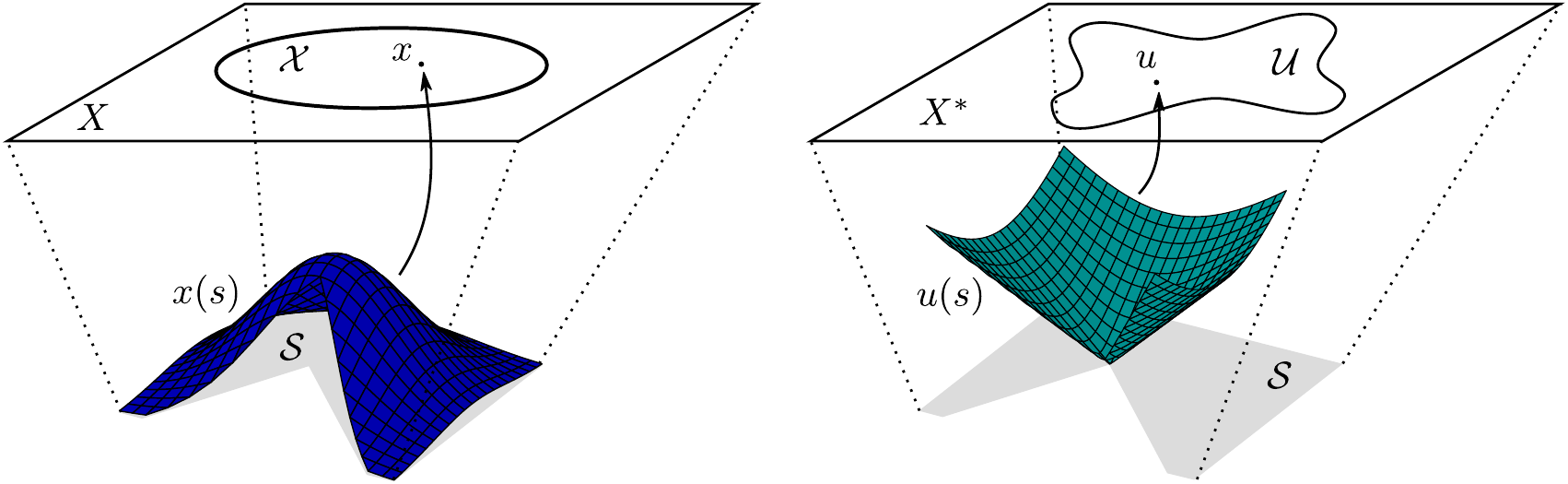}
\caption{Illustration of the lifting operation}
\label{fig:OOCUC:lifting}
\vspace{-2ex}
\end{figure}

\subsection{An Upper Bound on the Worst-Case Regret}
\label{subsec:OOCUC:UpperBound}

Let $(S,d)$ be a compact metric space, and let $\mu$ be a Borel measure on $S$. Suppose that the reward vectors~$u_\tau$ are given by elements in $L^q(S,\mu)$, where $q>1$. Let $X = L^p(S,\mu)$, where $p$ and $q$ are H\"older conjugates, i.e., $\frac{1}{p}+\frac{1}{q} = 1$. Consider $\Xcal = \{x\in X \,\suchthat\, x \geq 0 \;\text{a.e.}, \|x\|_1=1\}$, the set of probability measures on~$S$ that are absolutely continuous w.r.t. to $\mu$ with $p$-integrable Radon-Nikodym derivatives. 
Moreover, denote by $\Zcal$ the class of non-decreasing $\chi: [0,\infty) \rightarrow [0,\infty]$ such that $\lim_{r\to0}\chi(r)=\chi(0)=0$. 
The following assumption will be made throughout this section:

\begin{assumption}
\label{ass:OOCUC:LipschitzRewards}
The reward vectors $u_t$ have modulus of continuity $\chi$ on $S$, uniformly in~$t$. That is, there exists $\chi \in \Zcal$ such that $| u_t(s) - u_t(s')| \leq \chi(d(s, s'))$ for all $t$ and for all $s,s'\in S$.
\end{assumption}

Denote by $B(s,r) = \{s'\in S \,\suchthat\, d(s,s')<r\}$ the open ball of radius $r$ centered at $s$ and by $\Bcal(s, \delta) \subset \Xcal$ the set of elements of~$\Xcal$ with support contained in~$B(s,\delta)$. Furthermore, let $D_{S} := \sup_{s,s'\in S} d(s,s')$ the diameter of~$S$. 
Then we have the following:

\begin{theorem}[Dual Averaging Regret on Metric Spaces with Uniformly Continuous Rewards]
\label{thm:OOCUC:DiscT}
\newline
Let $(S,d)$ be compact, and suppose that Assumption~\ref{ass:OOCUC:LipschitzRewards} holds. Let~$h$ be a uniformly essentially strongly convex regularizer on~$\Xcal$ with modulus~$\gamma$, and let $(\eta_t)_{t\geq 1}$ be a positive non-increasing sequence of learning rates. Then, under~\eqref{eq:RegMin:DiscTime:DTBounds:Strategy}, for any positive sequence $(\vartheta_t)_{t\geq 1}$, \\[-3ex]
\begin{align}
\label{eq:OOCUC:DiscT:Bound}
\Rcal_t  \leq \frac{\sup_{s\in S} \inf_{x\in \Bcal(s,\vartheta_t)}h(x) - \underline{h}}{\eta_t} + t\, \chi(\vartheta_t)  +  \sum_{\tau=1}^t \|u_\tau\|_* \, \tilde\gamma^{-1} \Bigl( \frac{\eta_{\tau-1} }{2}\|u_\tau\|_* \Bigr)
\end{align}
\end{theorem}

\begin{remark}
The sequence $(\vartheta_t)_{t\geq 1}$ that appears in Theorem~\ref{thm:OOCUC:DiscT} is not a parameter of the Dual Averaging algorithm, but rather a parameter in the regret bound. In particular, \eqref{eq:OOCUC:DiscT:Bound} holds true for any such positive sequence, and we will use this fact later on to obtain explicit bounds by instantiating~\eqref{eq:OOCUC:DiscT:Bound} with a particular choice of $(\vartheta_t)_{t\geq 1}$.
\end{remark}

It is important to realize that the infimum over $\Bcal(s,\vartheta_t)$ in~\eqref{eq:OOCUC:DiscT:Bound} may be infinite, in which case the bound is meaningless. This happens for example if~$s$ is an isolated point of a compact subset~$S\subset \Rbb^n$ and $\mu$ is the Lebesgue measure, in which case $\Bcal(s,\vartheta_t) = \emptyset$. 
However, under an additional regularity assumption on the measure~$\mu$ we can avoid such degenerate situations. 
\begin{definition}[$Q$-regularity~\citep{Heinonen:2015aa}]
\label{def:OOCUC:Qregularity}
A Borel measure~$\mu$ on a metric space $(S,d)$ is (Ahlfors) $Q$-regular if there exist $0 < c_0 \leq C_0<\infty$
such that for any open ball $B(s, r)$ \\[-4ex]
\begin{align}
c_0 r^Q \leq \mu(B(s, r)) \leq C_0 r^Q
\label{eq:OOCUC:Qregularity}
\end{align}
We say that $\mu$ is $r_0$-locally $Q$-regular if~\eqref{eq:OOCUC:Qregularity} holds for all $0 < r \leq r_0$.
\end{definition}

Intuitively, under an $r_0$-locally $Q$-regular measure, the mass in the neighborhood of any point of~$S$ is uniformly bounded from above and below. This will allow, at each iteration~$t$, to assign sufficient probability mass around the maximizer(s) of the cumulative reward function.

\begin{example}[Regularity of the Lebesgue measure $\lambda$]
\label{ex:OOCUC:QRegularity}
The canonical example for a~$Q$-regular measure is the Lebesgue measure~$\lambda$ on $\Rbb^n$. If~$d$ is the metric induced by the standard Euclidean norm, then $Q=n$ and the bound~\eqref{eq:OOCUC:Qregularity} is tight with $c_0=C_0$, a dimensional constant. However, for general sets $S\subset\Rbb^n$, $\lambda$ need not be locally $Q$-regular. This is for example the case if $S$ includes isolated points. One sufficient condition for local regularity of the Lebesgue measure is that $S$ is $v$-uniformly fat, as defined by~\cite{Krichene:2015ab}, i.e., that for all $s\in S$, there exists a convex set $K \subseteq S$ that contains $s$ with $\lambda(K) \geq v$. In this case one can to prove that $\lambda|_S$ is $r_0$-locally $n$-regular, with $r_0$, $c_0$ and $C_0$ depending on the geometry of $S$. Importantly, $S$ need not be convex or even connected.
\end{example}

\begin{assumption}
\label{ass:OOCUC:QRegularity}
The measure $\mu$ is $r_0$-locally $Q$-regular on $(S,d)$.
\end{assumption}

Recall that he worst-case regret $\Rcal_t$ is defined as the supremum of $R_t(x)$ over all elements of~$\Xcal$. In the setting of this section, we can in fact say more: 
\begin{proposition}
\label{prop:OOCUC:EquivOfSuprema}
Suppose Assumption~\ref{ass:OOCUC:QRegularity} holds. Then\\[-3ex]
\begin{align}
\Rcal_t \;&=\; \sup_{x\in \Xcal} R_t(x) \;=\; \sup_{x\in \Pcal} R_t(x)  \;=\; \sup_{s\in S}\; U_t(s) - \textstyle\sum_{\tau=1}^t \langle u_\tau, x_\tau \rangle
\end{align}
\end{proposition}

Under Assumption~\ref{ass:OOCUC:QRegularity}, $\Bcal(s,\vartheta_t) \neq \emptyset$ for all $s\in S$ and $\vartheta_t > 0$, and hence there is hope for a bound on $\inf_{x\in \Bcal(s,\vartheta_t)}h(x)$ that is uniform in $s$. To get explicit rates of convergence, we have to consider a more specific class of regularizers.

\subsection{Explicit Rates for $f$-Divergences on $L^p(S)$}
\label{subsec:OOCUC:fDiv}

In this section we consider a particular class of regularizers called $f$-divergences or Csisz\'ar divergences~\citep{Csiszar:1967aa} and provide explicit bounds on the worst-case regret. Following~\cite{audibert2014regret}, we define $\omega$-potentials and the associated $f$-divergence.

\begin{definition}[$\omega$-Potential]
\label{def:OOCUC:fDiv:omegaPot}
Let $\omega \leq 0$ and $a \in (-\infty, +\infty]$. A continuous increasing diffeomorphism $\phi: (-\infty, a) \rightarrow  (\omega, \infty)$, is an $\omega$-potential if $\lim_{z \to -\infty} \phi(z) = \omega$, $\lim_{z \to a} \phi(z) = +\infty$ and $\phi(0) \leq 1$. Associated to $\phi$ is the convex function $f_\phi: [0, \infty) \rightarrow \Rbb$ defined by 
\[
f_\phi(x) = \int_{1}^x \phi^{-1}(z)\,dz,
\]
and the $f_\phi$-divergence, defined by \\[-1.5ex]
\[
h_{\phi}(x) = \int_{S} f_\phi\bigl(x(s)\bigr)\, d\mu(s) + \ind{\Xcal}(x).
\]
where $\ind{\Xcal}$ is the indicator function of~$\Xcal$ (i.e. $\ind{\Xcal}(x) = 0$ if $x\in \Xcal$ and $\ind{\Xcal}(x) = +\infty$ if $x\notin \Xcal$).
\end{definition}

A remarkable fact is that for regularizers based on $\omega$ potentials, the Dual Averaging update~\eqref{eq:RegMin:DiscTime:DTBounds:Strategy} can be computed efficiently. More precisely, it can be shown~\citep{Krichene:2016aa} that the maximizer in this case has a simple expression in terms of the dual problem, and the problem of computing $x_{t+1} = Dh^*(\eta_t \sum_{\tau = 1}^t u_\tau)$ reduces to computing a scalar dual variable $\nu^*_t$. In Proposition~\ref{prop:CompChoice:omegaPot:FrechetExplicit} in Appendix~\ref{appsec:CompChoice} we provide a bound on $\nu^*_{t+1}$ that depends on the value of $\nu^*_t$ and other parameters of the problem. In practice, these bounds greatly speed up computation. 


Note that the measure $\mu$ plays the role of a design variable, and its choice will affect our bounds through the constants $Q, C_0$ and $c_0$. The problem of finding a ``good'' measure $\mu$ is a very interesting problem for future studies. For now we will assume that $\mu(S) =1$, which due to compactness of $S$ is without loss of generality if we want arbitrarily small balls to have finite measure. 

\begin{proposition}
\label{prop:OOCUC:fDiv}
Suppose that $\mu(S) = 1$, and that Assumption~\ref{ass:OOCUC:QRegularity} holds with constants $r_0 > 0$ and $0 < c_0 \leq C_0<\infty$. Under the Assumptions of Theorem~\ref{thm:OOCUC:DiscT}, with $h=h_\phi$ the regularizer associated to an $\omega$-potential~$\phi$, we have that, for any positive sequence $(\vartheta_t)_{t\geq 1}$ with $\vartheta_t \leq  r_0$, \\[-2.5ex]
\begin{align}
\label{eq:OOCUC:fDiv:Bound}
\frac{\Rcal_t}{t}  \leq \frac{\min(C_0 \vartheta_t^Q, \mu(S))}{t\, \eta_t}  f_\phi\bigl( c_0^{-1} \vartheta_t^{-Q} \bigr) + \chi(\vartheta_t) + \frac{1}{t}  \sum_{\tau=1}^t \|u_\tau\|_* \, \tilde\gamma^{-1} \Bigl( \frac{\eta_{\tau-1} }{2}\|u_\tau\|_* \Bigr) \\[-5ex] \nonumber
\end{align}
\end{proposition}

For particular choices of the learning rates $(\eta_t)_{t\geq 1}$ and the sequence $(\vartheta_t)_{t\geq 1}$, we can derive explicit regret rates. To intuit, suppose for simplicity that the reward functions are uniformly bounded in the dual norm. Then it is clear that in order for the last term in~\eqref{eq:OOCUC:fDiv:Bound} to vanish asymptotically, $\eta_t$ must be vanishing. Similarly, for the second term to vanish, $\vartheta_t$ must be vanishing as well. Thus, if both $(\eta_t)_{t\geq 1}$ and $(\vartheta_t)_{t\geq1}$ are decreasing sequences, their respective rates of decay must be carefully chosen so that the first term also vanishes. These tradeoffs will become clear in the statement of Corollary~\ref{cor:OOCUC:omegaPot:ExpGenHedge:Bound} and in the numerical examples presented in Appendix~\ref{appsec:NumRes}.

%


\subsection{Analysis for Entropy Dual Averaging (The Generalized Hedge Algorithm)}
\label{subsec:OOCUC:omegaPot}

\label{subsubsec:OOCUC:omegaPot:ExpGenHedge}

Taking $\phi(z) = e^{z-1}$, we have that $f_\phi(x) = \int_1^x \phi^{-1}(z)dz = x \log x$, and hence the regularizer is $h_\phi(x) = \int_S x(s) \log x(s) d\mu(s)$. Then $Dh^*(\xi)(s) = \frac{\exp \xi(s)}{\|\exp \xi(s)\|_1}$. This corresponds to a generalized version of the Hedge algorithm~\citep{Arora:2012aa,Krichene:2015ab}. The regularizer $h_\phi$ can be shown to be essentially strongly convex with modulus $\gamma(r) = \frac{1}{2}r^2$. 

\begin{corollary}[Regret Bound for Entropy Dual Averaging]
\label{cor:OOCUC:omegaPot:ExpGenHedge:Bound}
Suppose that $\mu(S)=1$, that $\mu$ is $r_0$-locally $Q$-regular with constants $c_0, C_0$, that $\|u_t\|_* \leq M$ for all~$t$, and that $\chi(r) = C_\alpha r^\alpha$ for $0<\alpha\leq 1$ (that is, the rewards are $\alpha$-H\"older continuous).  Then, under Entropy Dual Averaging, choosing $\eta_t = \eta \sqrt{\log t / t}$  with $\eta= \frac{1}{M} \bigl( \frac{C_0Q}{2c_0} \log(c_0^{-1}\vartheta^{-Q/\alpha}) + \frac{Q}{2\alpha} \bigr)^{1/2}$ and $\vartheta>0$, we have that\\[-2.25ex]
\begin{align}
\label{eq:OOCUC:omegaPot:ExpGenHedge:Bound}
\frac{\Rcal_t}{t} &\leq \biggl( 2M \sqrt{ \frac{2C_0}{c_0} \Bigl( \log(c_0^{-1}\vartheta^{-Q/\alpha}) + \frac{Q}{2\alpha} \Bigr) } + C_\alpha \vartheta \biggr) \sqrt{\frac{\log t}{t}}
\end{align} 
whenever $\sqrt{\log t / t} <r_0^\alpha \vartheta^{-1}$. 
\end{corollary}

One can further optimize over the choice of $\vartheta$ to obtain the best constant in the bound. 
 Note also that the case $\alpha = 1$ corresponds to Lipschitz continuity.

\subsection{A General Lower Bound}
\label{subsec:OOCUC:LowerBound}

We also prove the following general lower bound for any online algorithm:

\begin{theorem}[General Lower Bound]
\label{thm:OOCUC:LowerBound}
Let $(S,d)$ be compact, suppose that Assumption~\ref{ass:OOCUC:QRegularity} holds, and let $\chi \in \Zcal$. Then for any online algorithm, 
there exist a sequence $(u_\tau)_{\tau=1}^t$ of reward vectors $u_\tau \in X^*$ with $\|u_\tau \|_* \leq M$ and modulus of continuity $\chi_\tau < \chi$ such that\\[-2.5ex]
\begin{align}
\label{eq:thm:OOCUC:LowerBound}
\Rcal_t &\geq \frac{w(D_S)}{2\sqrt{2}} \sqrt{t}
\end{align}
where $w: \Rbb \rightarrow \Rbb$ is any function with modulus of continuity~$\chi$ such that $\|w(d(\,\cdot\,,s'))\|_q \leq M$ for some $s' \in S$ for which there exists $s\in S$ with $d(s,s') = D_{\!S}$. 
\end{theorem}

Maximizing the constant in~\eqref{eq:thm:OOCUC:LowerBound} is of interest in order to benchmark the bound against the upper bounds obtained in the previous sections. This problem is however quite challenging, and we will defer this analysis to future work. For H\"older-continuous functions, we have the following result:

\begin{proposition}[General Lower Bound for H\"older-Continuous Functions]
\label{prop:OOCUC:LowerBoundHoelder}
In the setting of Theorem~\ref{thm:OOCUC:LowerBound}, suppose that $\mu(S) =1 $ and that $\chi(r) = C_\alpha r^\alpha$ for some $0<\alpha\leq 1$. Then \\[-2.5ex]
\begin{align}
\label{eq:prop:OOCUC:LowerBoundHoelder}
\Rcal_t &\geq \frac{\min\bigl( C_\alpha^{1/\alpha} D_{\!S}^\alpha\,,\, M \bigr)}{2\sqrt{2}}  \sqrt{t}
\end{align}
\end{proposition}

Observe that, up to a $\sqrt{\log t}$ factor, the asymptotic rate of this general lower bound for any online algorithm matches that of the upper bound~\eqref{eq:OOCUC:omegaPot:ExpGenHedge:Bound} of Entropy Dual Averaging.

\subsection{Consistency of Dual Averaging}
\label{subsec:OOCUC:Consistency}

It is quite intuitive to see that Dual Averaging would recover the greedy algorithm as the regularizer~$h$ ``approaches a constant''. In the following, we make this intuition precise.

\begin{definition}[Consistency of a Sequence of Regularizers]
\label{def:OOCUC:ConsistentReg}
A sequence $(h_1, h_2,  \dotsc)$ of regularizers on~$\Xcal$ is consistent if there exists $C\in \Rbb$ such that $h_i(x) \rightarrow C$ as $i \rightarrow \infty$ for all $x\in \Xcal$.
\end{definition}

For $s\in S$, $A\subset S$, let $d(s,A) = \inf_{s'\in A}d(s,s')$. For $\delta > 0$, let $B_{\delta}^* := \{s\in S \suchthat d(s,S^*) < \delta\}$. Moreover, let $\nu|_{A}$ denote the restriction of $\nu\in \Pcal(S)$ to $A$.

\begin{proposition}
\label{prop:OOCUC:LimitThm:strong}
Suppose Assumption~\ref{ass:OOCUC:QRegularity} holds and that $(h_i)_{i\geq 1}$ is a sequence of regularizers that is consistent. Fix $t$ and let $U^*:= \max_{s\in S}U_t(s)$ and $S^* := \{s\in S \,\suchthat\, U_t(s) = U^*\}$. 
For $i\geq 1$ let $x_{t, i}^* := D h_i^*(U_t)$ 
Then, for any $\delta>0$, we have that $x_{t, i}^*|_{(B_\delta^*)^c} \rightarrow 0|_{(B_{\delta}^*)^c}$ (strongly) as $i \rightarrow \infty$. Equivalently, $\int_{S^*} x_i^*(s)\,ds \rightarrow 1$ as $i\rightarrow \infty$. 
\end{proposition}

Proposition~\ref{prop:OOCUC:LimitThm:strong} shows that if the sequence of regularizers is consistent, the optimizers, in the limit, collapse to distributions supported on the set of maximizers of $U_t$ (as illustrated numerically in Example~\ref{ex:NumRes:Greedy} in Appendix~\ref{appsec:NumRes}). If the maximizer of $U_t$ is unique, we can say the following:
\begin{corollary}
\label{cor:OOCUC:LimitThm:weak}
In the setting of Proposition~\ref{prop:OOCUC:LimitThm:strong}, suppose that $U_t$ admits a unique maximizer $s_t^*\in S$. Then $x^*_i$ weakly converges to the Dirac measure on $s_t^*$ as $i \rightarrow \infty$. We write $x_{t, i}^* \rightharpoonup \delta_{s_t^*}$.
\end{corollary}

\section{Learning in Continuous Two-Player Zero-Sum Games}
\label{sec:LearnCont}

In this section we apply our esults on Dual Averaging on $L^p$-spaces in the context of repeated play of continuous games. In particular, we focus on continuous two-player zero-sum games. In the finite case similar results exist for non-zero sum games, and we believe that they can be extended to our setting, however this is outside the scope of this article (see for example~\citep{Stoltz:2007aa} for related work on learning correlated equilibria under additional convexity assumptions).

\subsection{Static Two-Player Zero-Sum Games}
\label{subsec:LearnCont:TPCS}

Consider a two-player zero sum game $\Gcal = (S_1,S_2,u)$, in which the strategy spaces $S_1$ and $S_2$ of player~$1$ and~$2$, respectively, are Hausdorff spaces, and $u : S_1\times S_2 \rightarrow \Rbb$ is the payoff function of player~$1$. As the game is zero-sum, the payoff function of player~$2$ is $-u$. For each~$i$, denote by $\Pcal_i:=\Pcal(S_i)$ the set of Borel probability measures on~$S_i$. Denote $S := S_1\times S_2$ and $\Pcal := \Pcal_1\times \Pcal_2$. For a (joint) mixed strategy $x \in \Pcal$, we define the natural extension $\bar{u}: \Pcal \rightarrow \Rbb$ by $\bar{u}(x) := \Ebb_x[u] = \int_{S} u(s^1,s^2)\, dx(s^1,s^2)$, which is the expected payoff of player~$1$ under~$x$. 


A continuous zero-sum game~$\Gcal$ is said to have \emph{value} $V$ if 
\begin{align}
\adjustlimits \sup_{x^1\in \Pcal_1} \inf_{x^2\in \Pcal_2} \bar{u}(x^1,x^2) = \adjustlimits \inf_{x^2\in \Pcal_2} \sup_{x^1\in \Pcal_1} \bar{u}(x^1,x^2) = V
\label{eq:LearnCont:TPCS:NashEqSaddle}
\end{align}
The elements $x^1\times x^2 \in \Pcal$ at which~\eqref{eq:LearnCont:TPCS:NashEqSaddle} holds are the (mixed) Nash Equilibria of~$\Gcal$. We denote the set of Nash equilibria of $\Gcal$ by $\Ncal(\Gcal)$.
%
In the case of finite games, it is well known that every two-player zero-sum game has a value. This is not true in general for continuous games, and additional conditions on strategy sets and payoffs are required. A classic result is the following:
\begin{theorem}[\citealp{Glicksberg:1950aa}]
\label{thm:LearnCont:TPCS:Glicksberg1950}
Let $S_1$ and $S_2$ be compact, and suppose that $u: S_1\times S_2 \rightarrow \Rbb$ is semi-continuous (upper or lower). Then $\Gcal$ has a value. 
\end{theorem} 


%
%
%

\subsection{Repeated Play}
\label{subsec:LearnCont:Repeated}

We consider repeated play of the continuous two-player zero-sum game. 
Given a game $\Gcal$ and a sequence of plays $(s_t^1 )_{t\geq 1}$ and $(s_t^2)_{t\geq 1}$, we say that player $i$ has sublinear (realized) regret if\\[-3ex]
\begin{align}
\label{eq:LearnCont:Repeated:HannanConistency}
\limsup_{t\rightarrow \infty} \frac{1}{t} \biggl(\; \sup_{s^i \in S_i} \sum_{\tau=1}^t u_i(s^i, s^{-i}_\tau) - \sum_{\tau=1}^t u_i(s^i_\tau, s^{-i}_\tau) \biggr) = 0
\end{align}
where we use $-i$ to denote the other player (e.g.  $-1 = 2$). 

A strategy~$\sigma^i$ for player~$i$ is, loosely speaking, a (possibly random) mapping from past observations to its actions. Of primary interest to us are Hannan-consistent strategies~\citep{Hannan:1957aa}:

\begin{definition}[Hannan Consistency]
\label{def:LearnCont:Repeated:HannanConistency}
A strategy $\sigma^i$ of player~$i$ is Hannan consistent if, for any sequence $(s_{-i}^t )_{t\geq 1}$, the sequence of plays $(s_{i}^t )_{t\geq 1}$ generated by $\sigma^i$ has sublinear regret almost surely.
%
%
\end{definition}
%


Note that the almost sure statement in Definition~\ref{def:LearnCont:Repeated:HannanConistency} is with respect to the randomness in the strategy~$\sigma^i$.  
The following results are generalizations of their counterparts for discrete games:

\begin{proposition}
\label{prop:LearnCont:Repeated:limsupV}
Suppose $\Gcal$ has value $V$ and consider a sequence of plays $(s^1_t)_{t\geq 1}$, $(s^2_t)_{t\geq 1}$ and suppose that player 1 has sublinear realized regret. Then \\[-3.5ex]
\begin{align}
\label{eq:LearnCont:Repeated:limsupV}
\liminf_{t\rightarrow \infty} \frac{1}{t} \sum_{\tau=1}^t u(s^1_\tau, s^2_\tau) \geq V
\end{align}
\end{proposition}

\begin{corollary}
\label{cor:LearnCont:Repeated:Value}
Suppose $\Gcal$ has value $V$ and consider a sequence of plays $(s^1_t)_{t\geq 1}$, $(s^2_t)_{t\geq 1}$ and assume that both players have sublinear realized regret. Then \\[-3.5ex]
\begin{align}
\label{eq:LearnCont:Repeated:Value}
\lim_{t\rightarrow \infty} \frac{1}{t} \sum_{\tau=1}^t u(s^1_\tau, s^2_\tau) = V
\end{align}
\end{corollary}

As in the discrete case~\citep{Cesa-Bianchi:2006aa}, we can also say something about convergence of the empirical distributions of play to the set of Nash Equilibria. Since these distributions have finite support for every $t$, we can at best hope for convergence in the weak sense as follows: 

\begin{theorem}[Weak Convergence of the Empirical Distributions of Play]
\label{thm:LearnCont:Repeated:ConvOfEmpPlay}
Suppose that in a repeated two-player zero sum game $\Gcal$ that has a value both players follow a Hannan-consistent strategy, and denote by $\hat{x}^i_t = \frac{1}{t}\sum_{\tau=1}^t \delta_{s^i_\tau}$ the marginal empirical distribution of play of player~$i$ at iteration~$t$. Let $\hat{x}_t := (\hat{x}^1_t, \hat{x}^2_t)$. Then $\hat{x}_t \rightharpoonup \Ncal(\Gcal)$ almost surely, that is, with probability~1 the sequence $(\hat{x}_t)_{t\geq1}$ weakly converges to the set of Nash equilibria of~$\Gcal$.
\end{theorem}

\begin{corollary}
\label{cor:LearnCont:Repeated:ConvOfEmpPlayUnique}
If $\Gcal$ has a unique Nash equilibrium $x^*$, then with probability~1, $\hat{x}_t \rightharpoonup x^*$.
\end{corollary}

\subsection{Hannan-Consistent Strategies}
\label{subsec:LearnCont:HannanStrategies}

By Theorem~\ref{thm:LearnCont:Repeated:ConvOfEmpPlay}, if each player follows a Hannan-consistent strategy, then the empirical distributions of play weakly converge to the set of Nash equilibria of the game. But do such strategies exist? Regret minimizing strategies are intuitive candidates, and the intimate connection between regret minimization and learning in games is well studied for special cases such as for finite games~\citep{Cesa-Bianchi:2006aa} or potential games~\citep{Monderer:1996aa}. Using our results from Section~\ref{sec:OOCUC}, we will show that, under an additional assumption on the underlying information structure, no-regret learning based on Dual Averaging leads to Hannan consistency in our setting.

Suppose that after each iteration~$t$, each player~$i$ observes a \emph{partial payoff function} $\tilde{u}^i_t : S_i \rightarrow \Rbb$ describing their payoff as a function of only their own action, $s_i$, holding the action played by the other player fixed. That is,\\[-3ex]
\begin{align}
\label{eq:LearnCont:Repeated:ObservedPayoff}
\tilde{u}^1_t (s^1) := u(s^1, s^2_t) &&  \tilde{u}^2_t (s^2) := -u(s^1_t, s^2)
\end{align}

\begin{remark}
Note that we do not assume that the players have knowledge of the the joint utility function $u$. However, we do assume that the player has full information feedback, in the sense that they observe partial reward functions $u(\,\cdot\,, s_\tau^{-i})$ on their entire action set, as opposed to only observing the reward  $u(s_\tau^1, s_\tau^2)$ of the action played (the latter corresponds to the bandit setting).
\end{remark}

We denote by $\tilde{U}^i_t = \{\tilde{u}^i_\tau\}_{\tau=1}^{t}$ the sequence of partial payoff functions observed by player~$i$. We use $\Ucal^i_t$ to denote the set of all possible such histories, and define $\Ucal^i_0 := \emptyset$. A strategy~$\sigma^i$ of player~$i$ is a collection $\{\sigma^i_t\}_{t=1}^\infty$ of (possibly random) mappings $\sigma^i_t : \Ucal^i_{t-1} \rightarrow S_i$, such that at iteration $t$, player~$i$ plays $s^i_t =  \sigma^i_t(U^i_{t-1})$. 
%
We make the following assumption on the payoff function: 
\begin{assumption}
\label{ass:LearnCont:HannanStrategies:UnifContUnif}
The payoff function $u$ is uniformly continuous in $s^i$ with modulus of continuity independent of $s^{-i}$ for $i=1,2$. That is, for each $i$ there exists $\chi^i \in \Zcal$ such that $|u(s,s^{-i}) - u(s',s^{-i})| \leq \chi^i(d_i(s,s'))$ for all $s^{-i} \in S_{-i}$. 
\end{assumption}

It is easy to see that Assumption~\ref{ass:LearnCont:HannanStrategies:UnifContUnif} implies that the game has a value (see e.g. the argument in the proof of Lemma~\ref{lem:pf:thm:LearnCont:Repeated:ConvOfEmpPlay}).
It also makes our setting compatible with that of our Dual Averaging algorithm from Section~\ref{sec:OOCUC}. Suppose therefore that each player randomizes their play according to the sequence of probability distributions on $S_i$ generated by Dual Averaging with regularizer~$h_i$. That is, suppose that for $i\in \{1,2\}$, $\sigma_t^i$ is a random variable with the following distribution:\\[-3ex]
\begin{align}
\label{eq:LearnCont:HannanStrategies:RandomizedPlay}
\sigma^i_t \sim Dh_i^*\bigl( \eta_{t-1} \textstyle \sum_{\tau =1}^{t-1} \tilde{u}^i_\tau \bigr).
\end{align}

\begin{theorem}
\label{thm:LearnCont:HannanStrategies:MainConsistency}
Suppose that player~$i$ uses strategy~$\sigma^i$ according to~\eqref{eq:LearnCont:HannanStrategies:RandomizedPlay}. If the Dual Averaging algorithm ensures sublinear regret (i.e. $\limsup_t \Rcal_t/t \leq 0$), then~$\sigma^i$ is Hannan-consistent. 
\end{theorem}

\begin{corollary}
\label{cor:LearnCont:HannanStrategies:BothPlayDA}
If both players use strategies according to~\eqref{eq:LearnCont:HannanStrategies:RandomizedPlay} with the respective Dual Averaging ensuring that $\limsup_t \Rcal_t/t \leq  0$, then with probability~1 the sequence $(\hat{x}_t)_{t\geq1}$ of empirical distributions of play weakly converges to the set of Nash equilibria of~$\Gcal$.
\end{corollary}

Interestingly, even though Dual Averaging is performed on $L^p(S_i)$, a strict subset of $\Pcal(S_i)$, Corollary~\ref{cor:LearnCont:HannanStrategies:BothPlayDA} still ensures weak convergence of the empirical distributions of play to $\Ncal(\Gcal)$.

\section{Examples}
\label{sec:Examples}

\subsection{A Game With Unique Mixed Strategy Equilibrium}
\label{subsec:Examples:UniqueMixed}

Consider the zero-sum game~$\Gcal_1$ between two players playing on the unit interval $S_i = [0,1]$ with payoff function given by \\[-4ex]
\begin{align}
u(s^1, s^2) = \frac{(1+s^1)(1+s^2)(1-s^1s^2)}{(1+s^1s^2)^2}
\label{eq:Examples:UniqueMixed:Utility}
\end{align}
Since $|D_{s^i}u| \leq 8$ for any $s^{-i}\in [0,1]$ the payoff function is Lipschitz. 
It can be shown that~$V = 4/\pi$ and that this game has no pure and a unique mixed Nash equilibrium, with equilibrium density $x^i(s) = \frac{2}{\pi \sqrt{s} (1+s)}$ the same for both players~\citep{Glicksberg:1953aa}. Note that $x^i$ is unbounded and that $x^i\in L^p(S_i,\lambda)$ for any $1\leq p < 2$. This unboundedness is the reason for the slow convergence of the empirical distributions to~$x^i$ near zero that we can observe in Figure~\ref{fig:Examples:UniqueMixed:histogram}. 

\begin{figure}[h]
\centering
\includegraphics[width=\textwidth]{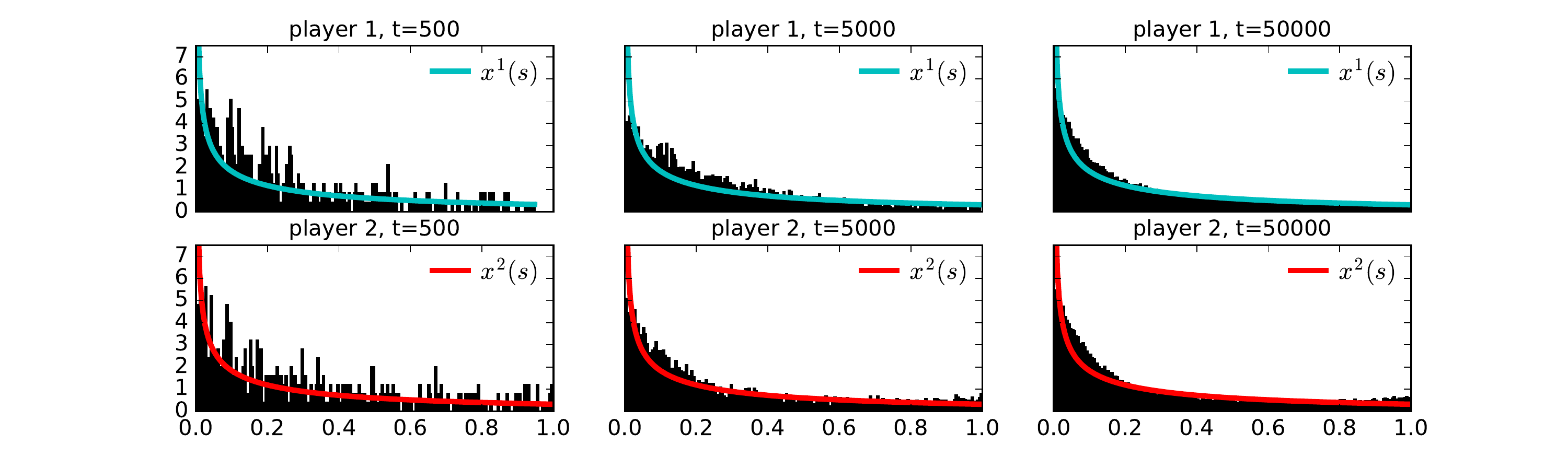}
\vspace{-4ex}
\caption{Normalized histograms of the empirical distributions of play for~$\Gcal_1$ (200 bins)}
\label{fig:Examples:UniqueMixed:histogram}
\vspace{-2ex}
\end{figure}

\subsection{A Game With Explicit Dual Averaging Updates}
\label{subsec:Examples:GameE}

Consider a zero-sum game~$\Gcal_2$ between two players on the unit interval with payoff function
\[
u(s^1, s^2) = s^2s^2 - a^1s^1  - a^2s^2
\]
where $a^1 = \frac{e-2}{e-1}$ and $a^2 =  \frac{1}{e-1}$. 
It is easy to verify that the pair $(x^1, x^2)$ given by $x^1(s) = \frac{\exp(s)}{e-1}$ and $x^2(s) = \frac{\exp(1-s)}{e-1}$ is a mixed-strategy Nash equilibrium of $\Gcal_2$. 
%
%
%
For sequences $(s^1_\tau)_{\tau=1}^t$  and $(s^2_\tau)_{\tau=1}^t$, the cumulative payoff functions for fixed action $s\in [0,1]$ are given, respectively, by
\begin{align*}
U^1_t(s^1) = \bigl( \Sigma_{\tau=1}^t s^2_\tau - a^1 t \bigr) s^1 - a^2 \Sigma_{\tau=1}^t s^2_\tau &&
U^2_t(s^2) = \bigl( a^2 t - \Sigma_{\tau=1}^t s^1_\tau \bigr) s_2 - a^1 \Sigma_{\tau=1}^t s^1_\tau
\end{align*}
%
%

If each player~$i$ uses the Generalized Hedge Algorithm with a sequence of learning rates $(\eta_\tau)_{\tau=1}^t$ to minimize their respective regret, then their strategy in period~$t$ is given by sampling from the distribution $x^i_{t}(s) \propto \exp(\alpha^i_{t}s)$, 
%
%
where $\alpha^1_{t} = \eta_t \bigl(\Sigma_{\tau=1}^t s^2_\tau - a^1 t \bigr)$ and $\alpha^2_{t} = \eta_t \bigl(a^1 t - \Sigma_{\tau=1}^t s^1_\tau \bigr)$. Interestingly, in this case the sum of the opponent's past plays is a sufficient statistic, in the sense that it completely determines the mixed strategy at time $t$.

\begin{figure}[h]
\centering
\includegraphics[width=\textwidth]{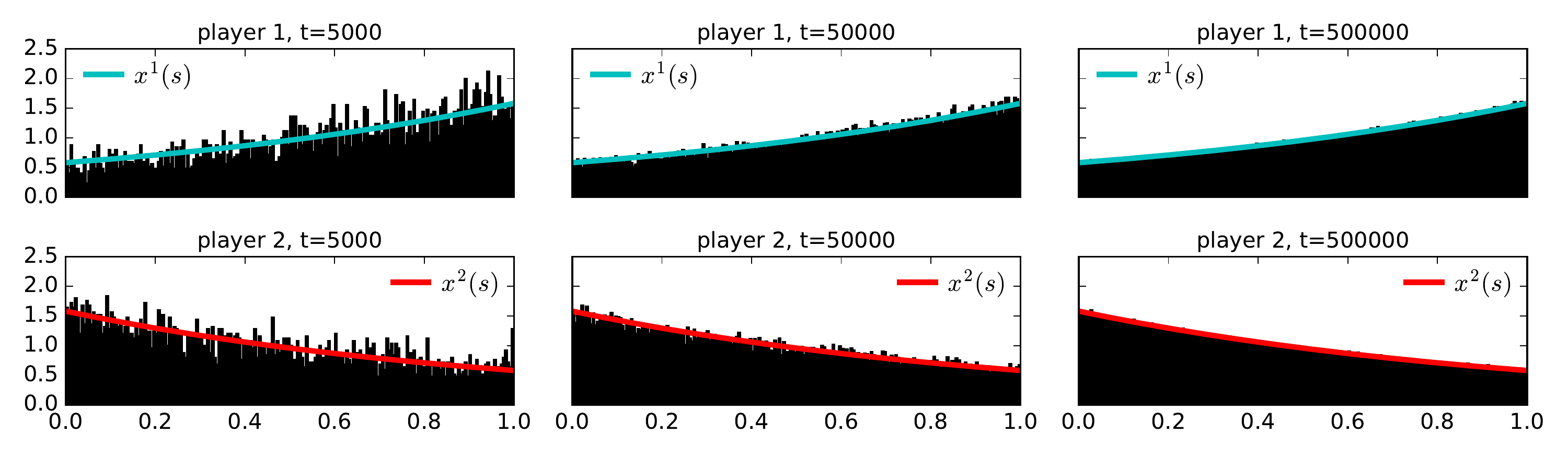}
\vspace{-4.25ex}
\caption{Normalized histograms of the empirical distributions of play in~$\Gcal_2$ (100 bins)}
\label{fig:Examples:GameE:histogram}
\end{figure}

Figure~\ref{fig:Examples:GameE:histogram} shows normalized histograms of the empirical distributions of play at different iterations~$t$. As $t$ grows the histograms approach the equilibrium densities $x^1$ and $x^2$, respectively. Note however, that this does not mean that the individual strategies $x^i_t$ converge. Indeed, Figure~\ref{fig:Examples:GameE:alphas} shows that the parameters $\alpha^i_t$ keep oscillating around the equilibrium parameters $1$ and $-1$, respectively, even for very large~$t$. We do, however, observe that the time-averaged parameters $\bar{\alpha}^i_t$ converge to the equilibrium values $1$ and $-1$.  

\begin{figure}[h]
\centering
\includegraphics[width=0.8\textwidth]{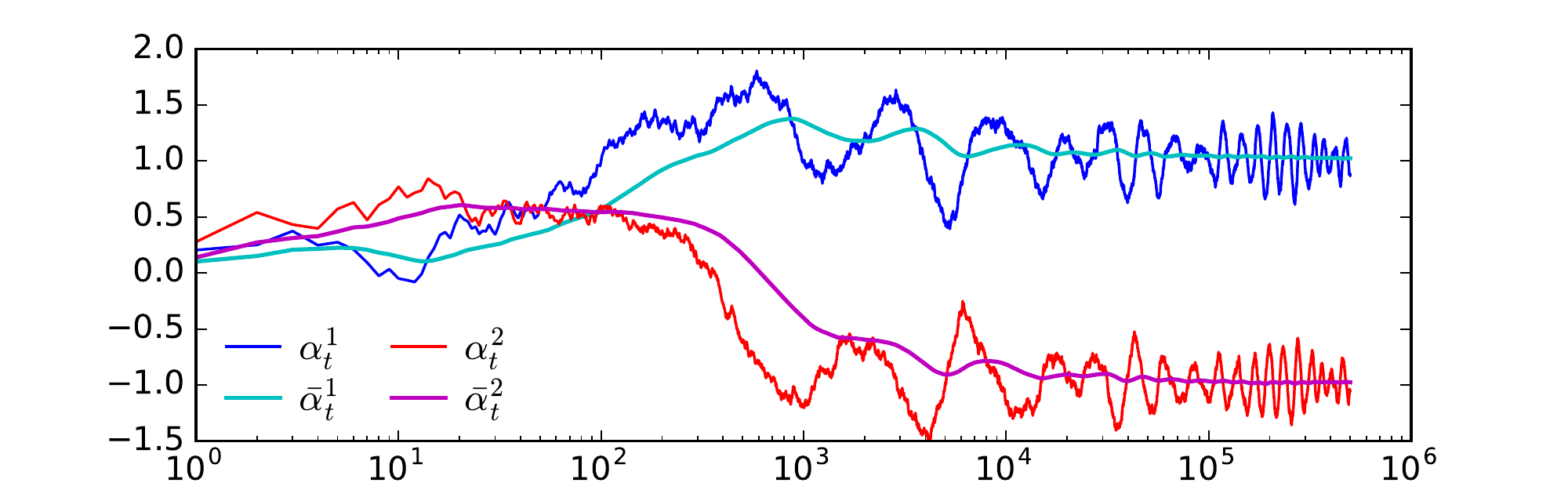}
\vspace{-1ex}
\caption{Evolution of parameters $\alpha^i_t$ and $\bar{\alpha}^i_t:=\frac{1}{t} \sum_{\tau=1}^t\alpha^i_\tau$ in~$\Gcal_2$}
\label{fig:Examples:GameE:alphas}
\end{figure}

\subsection{A Game on a Non-Convex Domain}
\label{subsec:Examples:NonConvex}

One of the most interesting features of the Dual Averaging algorithms discussed in Section~\ref{sec:OOCUC} is that they are applicable also in case of non-convex domains. We may therefore utilize them as a tool to compute approximate Nash equilibria in continuous zero-sum games on non-convex domains. In particular, consider a game $\Gcal_3$ in which each $S_i = [0,2]^2 \setminus [0.4,1]^2$ is an $L$-shaped subset of~$\Rbb^2$. It is easy to see that the Lebesgue measure on this set is $Q$-regular with $Q=2$, $c_0 = \frac{\pi}{4} $ and $C_0 = \pi$.
We define the metric $\tilde{d}$ on $S_1$ between any two points $a, b \in S_i$ as the length (in the Euclidean distance) of the shortest path between $a$ and $b$ that is entirely contained in $S_i$. The payoff function $u$ is given as $u(s^1,s^2) = \tilde{d}(s^1,s^2) - \frac{1}{10} \tilde{d}(s^1,0)$, which can be interpreted as a ``hide and seek'' game in which player~1 would like to get as far away from player~2 as possible, while at the same time having a preference for being near the origin. Player~2 instead wants to be as close to player~1 as possible. Intuitively, this game will not admit a pure Nash equilibrium.  Given the geometry of the problem, computing a mixed Nash equilibrium (whose existence follows from Theorem~\ref{thm:LearnCont:TPCS:Glicksberg1950}) poses a challenge. 
Instead, having both players play Entropy Dual Averaging on $L^p(S_i, \lambda)$, we observe in Figure~\ref{fig:Examples:NonConvex:regrets} that  they indeed incur sublinear regret, and that the empirical distributions of play do converge. Figure~\ref{fig:Examples:NonConvex:densities} shows Kernel Density Estimates (KDE) of $\hat{x}_t^1$ and $\hat{x}^2_t$ after $t=7500$ iterations. 

\begin{figure}[h]
\centering
\includegraphics[width=0.65\textwidth]{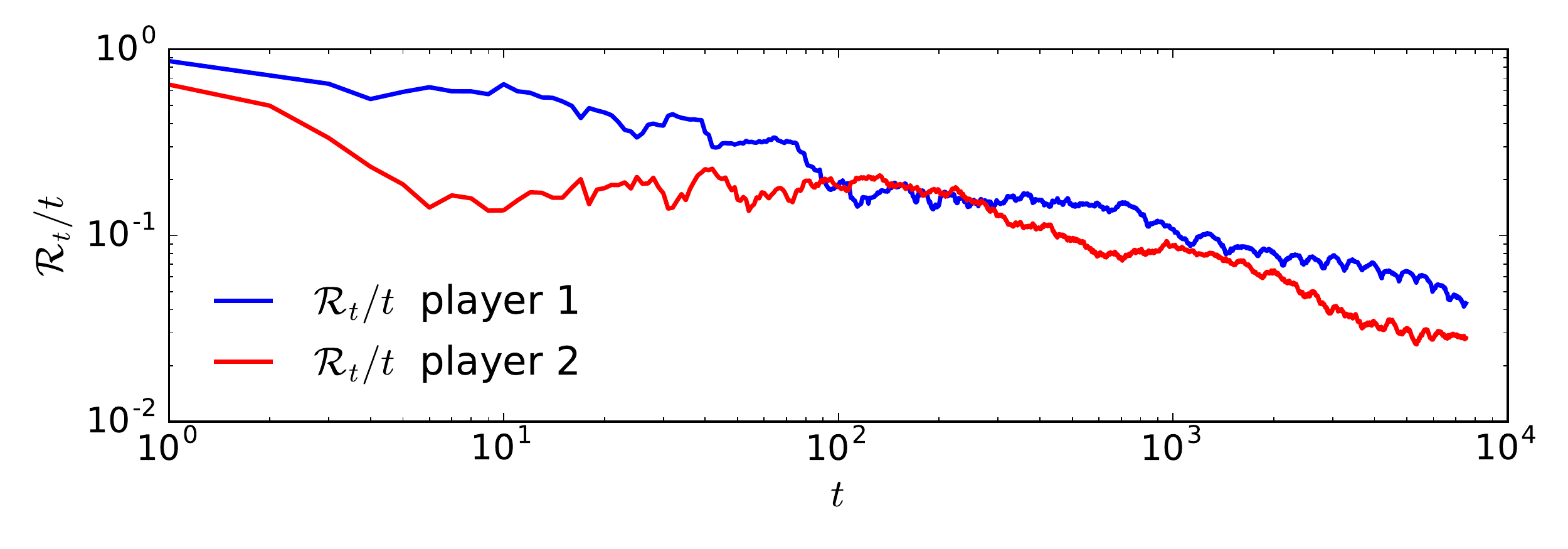}
\vspace{-2.5ex}
\caption{Time-Average Regrets (log-log scale) for Generalized Hedge in~$\Gcal_3$}
\label{fig:Examples:NonConvex:regrets}
\end{figure}

\begin{figure}[h]
\centering
\includegraphics[width=0.8\textwidth]{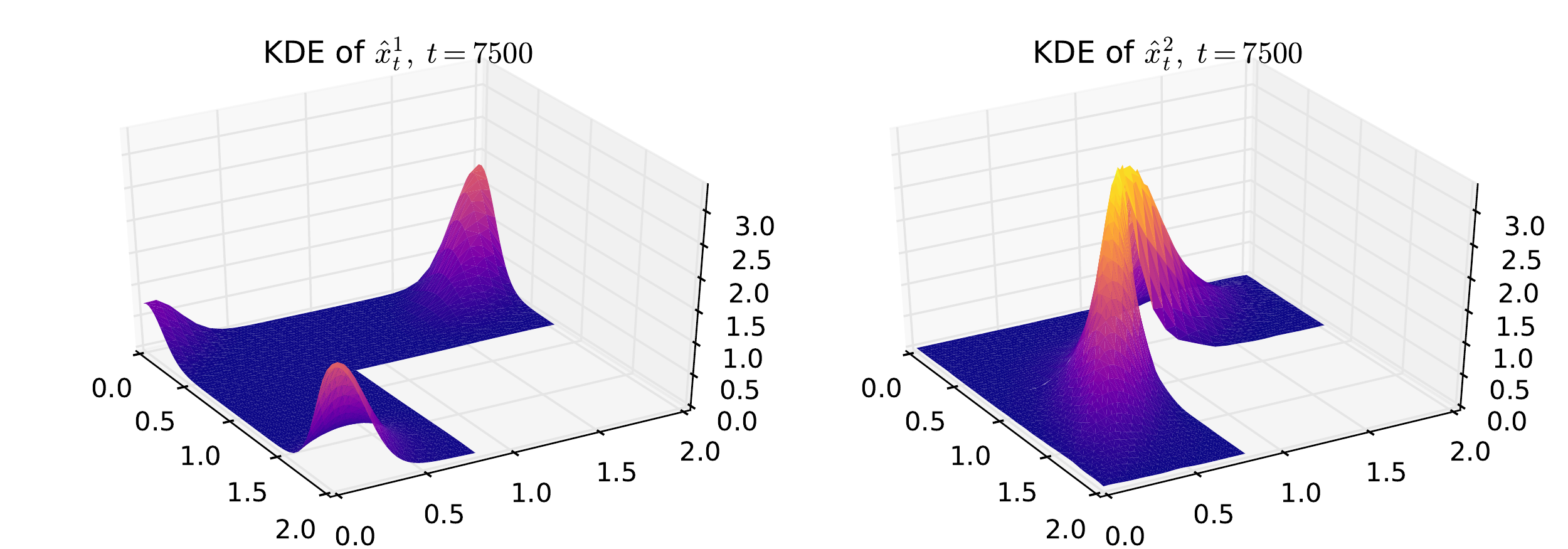}
\caption{Kernel Density Estimates of $\hat{x}^1_t$ (left) and $\hat{x}^2_t$ (right) in~$\Gcal_3$ for $t=7500$}
\label{fig:Examples:NonConvex:densities}
\end{figure}

\vspace{10ex}
{\small\bibliography{../references/ContNoRegret}

\appendix

\newpage
\section{Review of Some Results From Convex Analysis}
\label{appsec:ConvAnalysis}

In this section we collect some results from infinite-dimensional convex analysis that will play an important role in our analysis of the Dual Averaging algorithm.

\begin{lemma}[\citealp{Asplund:1968aa}]
\label{lem:RegMin:Prelims:PWEquiv}
Let $f: X \rightarrow (-\infty, +\infty]$ be proper lower semicontinuous. For a pair $(x_0, \xi_0) \in X \times X^*$ the following are equivalent:
\begin{enumerate}[label=(\roman*)]
\item $f^*$ is finite and Fr\'echet differentiable at $\xi_0$ with Fr\'echet derivative $Df^*(\xi_0) = x_0$. 
\label{def:RegMin:Prelims:PWEquiv:Fdiffable}
\item For some $\gamma^* \in \Gamma_{\!L}$, 
\begin{align}
f^*(\xi) \leq f^*(\xi_0) + \langle x_0, \xi - \xi_0\rangle + \gamma^*(\|\xi-\xi_0\|),\quad \forall\,\xi \in X^*
\label{eq:RegMin:Prelims:PWEquiv:StrongSmooth}
\end{align}
and $f^*(\xi_0) \in \Rbb$. 
\label{lem:RegMin:Prelims:PWEquiv:StrongSmooth}
\item For some $\gamma \in \Gamma_{\!U}$, 
\begin{align}
f(x) \geq f(x_0) + \langle x-x_0, \xi_0 \rangle + \gamma(\|x - x_0\|),\quad \forall\, x \in X
\label{eq:RegMin:Prelims:PWEquiv:StrongConvex}
\end{align}
and $f(x_0) \in \Rbb$. 
\label{lem:RegMin:Prelims:PWEquiv:StrongConvex}
\item $f^*$ is finite at $\xi_0$, $\dom f^*$ is radial at $\xi_0$, and $x_j \rightarrow x_0$ in norm whenever 
\begin{align}
\lim_{j\rightarrow \infty} \langle x_j, \xi_0 \rangle - f(x_j) = f^*(\xi_0)
\label{eq:RegMin:Prelims:PWEquiv:Limit}
\end{align}
\label{lem:RegMin:Prelims:PWEquiv:Limit}
\end{enumerate}
\vspace{-3ex}
Any of the above conditions implies that $\langle x_0, \xi_0\rangle = f(x_0) + f^*(\xi_0)$ (in other words: the Fenchel-Young inequality holds with equality) and that $f(x_0) = f^{**}(x_0)$. The functions $\gamma$ and $\gamma^*$ in~\ref{lem:RegMin:Prelims:PWEquiv:StrongSmooth} and~\ref{lem:RegMin:Prelims:PWEquiv:StrongConvex} form a pair of mutually dual functions.
\end{lemma}
Note that the function $f$ in Lemma~\ref{lem:RegMin:Prelims:PWEquiv} need not be convex. The following result will be essential to our analysis:
\begin{theorem}[
\citealp{Stromberg:2011aa}]
\label{thm:RegMin:Prelims:MainTheorem:FrechetDiffability}
Let $f: X \rightarrow (-\infty, +\infty]$ be lower semicontinuous. Then $f^*$ is proper and essentially Fr\'echet differentiable if and only if $f$ is a convex proper function that is essentially strongly convex. 
\end{theorem}

\section{Dual Averaging in Continuous Time}
\label{appsec:DAContTime}

In this section we use ideas from~\cite{Kwon:2014aa} and introduce a continuous-time regret minimization problem related to the one in discrete-time discussed in Section~\ref{subsec:DAContTime}. In fact, this analysis will be crucial in proving the discrete-time regret bound~\eqref{eq:RegMin:DiscTime:DTBounds:Strategy} in Theorem~\ref{thm:RegMin:DiscTime:DTBounds}.

\subsection{Regret Minimization in Continuous Time on Reflexive Banach Spaces}
\label{appsubsec:DAContTime:General}

Consider a reflexive Banach space $X$ with dual~$X^*$ and regularizer~$h$ on~$\Xcal$. Furthermore, suppose that $u^c: [0,\infty) \rightarrow X^*$ is a continuous-time reward process satisfying the following assumptions:
\begin{assumption}
\label{ass:DAContTime:LocalIntegrability}
The reward process $u^c$ is locally  integrable for any $x\in X$. That is, for all $x\in X$, $r_x: t\mapsto \langle u_t^c, x \rangle$ is Lebesgue-integrable on any compact set $K\subset [0,\infty)$.
\end{assumption}
%
\begin{assumption}
\label{ass:DAContTime:dualbound}
There exists $M< \infty$ such that $\sup_{x\in\Xcal} | \langle u^c_t, x \rangle | \leq M$ for all $t$.
\end{assumption}

Let $\eta^c : [0,\infty) \rightarrow (0,\infty)$ be a non-increasing and piece-wise continuous learning rate process. Furthermore, let $U^c_t = \int_0^t u^c_\tau \,d\tau$ be the cumulative reward function. We consider the continuous-time process $x^c:[0,\infty) \rightarrow X$ given by 
\begin{align}
\label{eq:DAContTime:Algo}
x^c_t := Dh^*(\eta^c_t \,U^c_t) 
\end{align}

\begin{theorem}[Continuous-Time Regret Bound]
\label{thm:DAContTime:MainBound}
Let $h$ be a regularizer function on $\Xcal$, let $\eta^c$ be non-decreasing and locally piecewise continuous. Suppose that the reward process~$u^c$ satisfies Assumptions~\ref{ass:DAContTime:LocalIntegrability} and~\ref{ass:DAContTime:dualbound}. Then under~\eqref{eq:DAContTime:Algo} we have, for any $x \in X$, that 
\begin{align}
\label{eq:DAContTime:MainBound}
R_t^c(x) := \int_0^t \langle u^c_\tau, x \rangle\,d\tau - \int_0^t \langle u^c_\tau, x^c_\tau \rangle\,d\tau \leq \frac{h(x)-\underline{h}}{\eta^c_t} 
\end{align}
where $\underline{h} := \inf_{x\in \Xcal} h(x)$. 
\end{theorem}

\begin{proof}[Theorem~\ref{thm:DAContTime:MainBound}]
Let $y^c_t = \eta^c_t \int_0^t u^c_\tau \,d\tau$. By linearity,
%
%
\begin{align*}
\eta^c_t \int_0^t \! \langle u^c_\tau, x \rangle \, d\tau = \eta^c_t \bigl\langle \textstyle\int_0^t u^c_\tau\, d\tau, x \bigr\rangle  = \langle y^c_t, x \rangle
\end{align*}
Assume for now that $\eta^c \in C^1$. If $h$ is proper, then 
\begin{align}
\label{eq:DAContTime:MainBound:step1}
\int_0^t \langle u^c_\tau, x \rangle \, d\tau = \frac{\langle y^c_t, x\rangle}{\eta^c_t} \leq \frac{h^*(y^c_t) + h(x)}{\eta^c_t} = \frac{h^*(y^c_t)}{\eta^c_t} + \frac{h(x)}{\eta^c_t}
\end{align}
by the Fenchel-Young inequality. By Theorem~\ref{thm:RegMin:Prelims:MainTheorem}, $h^*$ is essentially Fr\'echet differentiable with Fr\'echet gradient $D h^*(y)$. Furthermore, $y^c_t$ is differentiable. Thus, applying the chain rule 
and using that $x^c_t = Dh^*(y^c_t) = \argmax_{x\in\Xcal} \bigl( \langle x, y^c_t \rangle - h(x) \bigr) = \argmax_{x\in X} \bigl( \langle x, y^c_t \rangle - h(x) \bigr)$ we obtain
\begin{align*}
\frac{d}{dt} \frac{h^*(y^c_t)}{\eta^c_t} &= \frac{\eta^c_t \langle D h^*(y^c_t), \frac{d}{dt} y^c_t \rangle - h^*(y^c_t) \, \dot{\eta}^c(t)}{(\eta^c_t)^2} \\
&= \frac{\bigl\langle x^c_t , \bigl(\eta^c_t u^c_t + \dot{\eta}^c_t \int_0^t u^c(s,\tau)\,d\tau \bigr)\bigr\rangle}{\eta^c_t} - \frac{\dot{\eta}^c_t}{(\eta^c_t)^2} h^*(y^c_t) \\
&= \langle x^c_t, u^c_t \rangle + \frac{\dot{\eta}^c_t}{(\eta^c_t)^2} \bigl( \langle x^c_t,y^c_t \rangle - h^*(y^c_t) \bigr) \\
&= \langle x^c_t, u^c_t \rangle + \frac{\dot{\eta}^c_t}{(\eta^c_t)^2} h(x^c_t)
\end{align*}
Now $\dot{\eta}^c_t \leq 0$ by assumption, and hence
\begin{align*}
\frac{d}{dt} \frac{h^*(y^c_t)}{\eta^c_t} \leq  \langle x^c_t, u^c_t \rangle + \frac{\dot{\eta}^c_t}{(\eta^c_t)^2} \, \underline{h}
\end{align*}
Integrating from $t=0$ to $t=t$ yields
\begin{align*}
\frac{h^*(y^c_t)}{\eta^c_t} - \frac{h^*(y^c_0)}{\eta^c(0)} &\leq \int_0^t \langle x^c_\tau, u^c_\tau \rangle \,d\tau + \underline{h}   \int_0^t \frac{\dot{\eta}^c_\tau}{(\eta^c_\tau)^2} d\tau \\
&= \int_0^t \langle x^c_\tau, u^c_\tau \rangle \,d\tau - \underline{h} \biggl( \frac{1}{\eta^c_t} - \frac{1}{\eta^c_0} \biggr) 
\end{align*}
Now $y^c_0 = 0$, and hence $h^*(y^c_0) = \sup_{x\in X} - h(x) = -\underline{h}$, and so
\begin{align*}
\frac{h^*(y^c_t)}{\eta^c_t} \leq \int_0^t \langle x^c_\tau, u^c_\tau \rangle \,d\tau -  \frac{\underline{h}}{\eta^c_t}
\end{align*}
Plugging this into~\eqref{eq:DAContTime:MainBound:step1}, collecting terms and rearranging yields~\eqref{eq:DAContTime:MainBound}.  

Now suppose that $\eta^c$ is only piecewise continuous. Then there exists a sequence $(\eta^{c,i})_{i=1}^\infty$ of positive nonincreasing $C^1$ functions such that $\eta^{c,i} \rightarrow \eta^c$ pointwise a.e.. Let $x^{c,i}_t := Dh^*(\eta^{c,i}_t U^c_t)$. Note that $Dh^*$ is continuous by Theorem~\ref{thm:RegMin:Prelims:MainTheorem} and thus $x^{c,i}_t \rightarrow x^c_t$ pointwise. 
By Assumption~\ref{ass:DAContTime:dualbound} we have that $| \langle u^c_\tau, x^{c,i}_\tau \rangle | < M$ for all $\tau, i$ and thus $\int_0^t \langle u^c_\tau, x^{c,i}_\tau \rangle\,d\tau \rightarrow \int_0^t \langle u^c_\tau, x^c_\tau \rangle \,d\tau$ by Dominated Convergence.
\end{proof}

\subsection{Online Optimization in Continuous Time on Compact Metric Spaces}
\label{appsubsec:DAContTime:OOPTcompact}

One can also obtain bounds on the regret in continuous time by using similar arguments as in Section~\ref{sec:OOCUC}. While we do not make use of them in the main part of this article, these bounds may be of independent interest. 

We consider the setting of Section~\ref{sec:OOCUC}. Specifically, let $(S,d)$ be a compact metric space, and let $\mu \in \Pcal$, the set of Borel measures on~$S$. Denote by $B(s,r) = \{s'\in S \,\suchthat\, d(s,s')<r\}$ the open ball of radius $r$ centered at $s$. For $p>1$ consider $X = L^p(S,\mu)$ and $\Xcal = \{x\in X \,\suchthat\, x \geq 0 \;\text{a.e.}, \|x\|_1=1\}$, the set of probability measures on $S$ that are absolutely continuous w.r.t. to $\mu$ and whose Radon-Nikodym densities are $p$-integrable. Denote by $D_{S} := \sup_{s,s'\in S} d(s,s')$ the diameter of~$S$ and by $\Bcal(s, \vartheta^c_t) \subset \Xcal$ the set of elements of~$\Xcal$ with support contained in~$B(s,\vartheta^c_t)$. 
We need the following continuous-time variant of Assumption~\ref{ass:OOCUC:LipschitzRewards}:

\begin{assumption}
\label{ass:DAContTime:OOPTcompact:LipschitsRewards}
The reward process $u^c$ has modulus of continuity $\chi$ on $S$, uniformly in~$t$. That is, there exists $\chi \in \Zcal$ such that $| u^c_t(s) - u^c_t(s')| \leq \chi(d(s, s'))$ for all $s,s'\in S$ for all $t$.
\end{assumption}

\begin{theorem}[Continuous-Time Regret Bound on Metric Spaces]
\label{thm:DAContTime:OOPTcompact:OOCUC}
Let $(S,d)$ be compact, and suppose that Assumption~\ref{ass:DAContTime:OOPTcompact:LipschitsRewards} holds. Let~$h$ be a regularizer function on~$\Xcal$, and let $\eta^c$ be non-decreasing and locally piecewise continuous. Suppose further that $\vartheta^c: [0,\infty) \rightarrow (0,\infty)$ is a non-negative function and that the reward process~$u^c$ satisfies Assumptions~\ref{ass:DAContTime:LocalIntegrability} and~\ref{ass:DAContTime:dualbound}. Then, under the process~\eqref{eq:DAContTime:Algo}, 
\begin{align}
\label{eq:OOCUC:Bound}
\Rcal_t^c \leq \frac{\sup_{s\in S} \inf_{x\in \Bcal(s,\vartheta^c_t)}h(x)}{\eta^c_t} + t\, \chi(\vartheta^c_t) - \frac{\underline{h}}{\eta^c_t}
\end{align}
\end{theorem}

\begin{proof}[Theorem~\ref{thm:DAContTime:OOPTcompact:OOCUC}]
Similar to the proof of Theorem~\ref{thm:OOCUC:DiscT}.
\end{proof}

\begin{proposition}
\label{prop:DAContTime:OOPTcompact:fDiv}
Suppose that Assumption~\ref{ass:OOCUC:QRegularity} holds with constants $c_0>0$ and $C_0<\infty$. Under the Assumptions of Theorem~\ref{thm:DAContTime:OOPTcompact:OOCUC}, with essentially strongly convex regularizer $h_\phi$ the $f$-divergence of an $\omega$-potential~$\phi$, we have the following regret bound:
\begin{align}
\label{eq:DAContTime:OOPTcompact:fDiv:Bound}
\frac{\Rcal_t^c}{t} \leq  \frac{ \min( C_0\, (\vartheta^c_t)^Q, \mu(S)) }{t\, \eta^c_t}  f_\phi\bigl( c_0^{-1} (\vartheta^c_t)^{-Q} \bigr) + \chi(\vartheta^c_t)
\end{align}
\end{proposition}

\begin{proof}[Proposition~\ref{prop:DAContTime:OOPTcompact:fDiv}]
Similar to the proof of Proposition~\ref{prop:OOCUC:fDiv}.
\end{proof}

\section{Computing the Dual Averaging Optimizer}
\label{appsec:CompChoice}

In this section we discuss some aspects concerning the computation of the optimizer in the Dual Averaging update in the setting of online optimization on compact metric spaces with uniformly continuous rewards. The results of this section are used for generating the Hannan-consistent strategies in the repeated games in Section~\ref{sec:Examples}, and for performing the numerical benchmarks of the algorithms in Appendix~\ref{appsec:NumRes}.

As pointed out in Section~\ref{subsec:OOCUC:fDiv}, it can be shown that for $f$-Divergences of $\omega$-potentials, the Fr\'echet differential $Dh^*$ in this case has a simple expression in terms of the dual problem, and the problem of computing $x_{t+1} = Dh^*(\eta_t \sum_{\tau = 1}^t u_\tau)$ reduces to computing a scalar dual variable $\nu^*_t$. In particular, one can show the following:

\begin{proposition}[\citealp{Krichene:2016aa}]
\label{prop:CompChoice:omegaPot:FrechetExplicit}
Let $\phi$ be an $\omega$-potential with associated $f$-Divergence $h_\phi$ on~$\Xcal$. Then 
\begin{align}
\label{eq:CompChoice:omegaPot:FrechetExplicit}
Dh_\phi^*(\xi) = \phi(\xi + \nu^\star)_+
\end{align}
where $(\,\cdot\,)_+$ denotes the positive part of $(\,\cdot\,)$, and $\nu^\star$ satisfies $\int_S \phi(\xi + \nu^\star)_+\, d\mu(s) = 1$.
\end{proposition}

By Proposition~\ref{prop:CompChoice:omegaPot:FrechetExplicit}, the Fr\'echet derivative $Dh_\phi^*$ at $\xi = \eta_t U_t$ is entirely determined by the dual variable~$\nu^\star$, the unique $\nu$ such that $f(\nu)=1$, where $f(\nu) = \int_S \phi(\eta_{t}(U_t(s) + \nu^\star))_+ \, d\mu(s)$. Since $f$ is increasing by assumption on $\phi$, $\nu^\star$ can be determined using a simple bisection method. 
To guide the search for $\nu_t^\star$ for $t>0$ we can make use of the following result:
\begin{proposition}
\label{prop:CompChoice:omegaPot:nustarbounds}
Suppose $\phi$ is convex and let $\nu_t^\star$ the optimal dual variable determining $Dh_\phi^*(\eta_tU_t)$. Then 
\begin{align}
\label{eq:CompChoice:omegaPot:nustarbounds}
\frac{\eta_{t}}{\eta_{t+1}} \nu_{t}^\star - M \; \leq\;  \nu_{t+1}^\star \; \leq \;  \frac{\eta_t}{\eta_{t+1}} \nu_t^\star +  \frac{\eta_t-\eta_{t+1}}{\eta_{t+1}} \,tM 
\end{align}
where $\nu_0^\star = \eta_0^{-1} \phi^{-1}(1)$. Moreover, for $\eta_t = \eta\, t^{-\beta}$ this interval has length $\approx (1+\beta)M$. 
\end{proposition}

\begin{proof}[Proposition~\ref{prop:CompChoice:omegaPot:nustarbounds}]
Since $U_t \equiv 0$, we have $\nu_0^\star = \eta_0^{-1} \phi^{-1}(1)$. Moreover, by definition we have
\begin{align*}
\int_S \phi \bigl( \eta_t\bigl( U_t(s) + \nu_t^\star\bigr) \bigr)_+ \,d\mu(s) =  \int_S \phi \bigl( \eta_{t+1} \bigl( U_{t+1}(s) + \nu_{t+1}^\star\bigr) \bigr)_+\,d\mu(s) = 1
\end{align*}
If $\phi$ is convex, then so is $\phi(\,\cdot\,)_+$ as $z\mapsto z_+$ is convex and nondecreasing. Therefore
\begin{align*}
1 
&= \int_S \phi \Bigl( \eta_t \bigl( U_t(s) + \nu_t^\star\bigr) + (\eta_{t+1}-\eta_{t})U_t(s) - \eta_t \nu_t^\star + \eta_{t+1}\nu_{t+1}^\star + \eta_{t+1} u_{t+1}(s) \Bigr)_{\!+} \,d\mu(s) \\
&\leq \int_S \phi \Bigl( \eta_t\bigl( U_t(s) + \nu_t^\star\bigr) \Bigr)_{\!+} +  \phi \Bigl( \eta_t\bigl( U_t(s) + \nu_t^\star\bigr) \Bigr)_{\!+}' \Bigl((\eta_{t+1}-\eta_{t})U_t(s) \\
&\qquad\qquad\qquad\qquad\qquad\qquad \qquad \qquad - \eta_t \nu_t^\star + \eta_{t+1}\nu_{t+1}^\star + \eta_{t+1}u_{t+1}(s) \Bigr) d\mu(s) \\
&\leq 1 + \int_S \phi \bigl( \eta_t\bigl(U_t(s) + \nu_t^\star\bigr) \bigr)_{\!+}' \bigl(\eta_{t+1} \nu_{t+1}^\star - \eta_{t}\nu_{t}^\star + \eta_{t+1}M \bigr) d\mu(s)
\end{align*}
and hence, since $\phi' \geq 0$, we must have that $\eta_{t+1} \nu_{t+1}^\star - \eta_{t}\nu_{t}^\star + \eta_{t+1}M \geq 0$. Rearranging yields the lower bound on $\nu_{t+1}^\star$. The other inequality is proven in a similar fashion by reversing the roles of $t$ and $t+1$.
Finally, to show that the interval has length $\approx (1+\beta)M$ independent of $t$, note that $\frac{\eta_{t}}{\eta_{t+1}} = (1+\frac{1}{t})^\beta \approx 1+\frac{\beta}{t}$, and so $\frac{\eta_t-\eta_{t+1}}{\eta_{t+1}}tM \approx \beta M$.
\end{proof}

Having determined $\nu_t^\star$, we then have an explicit form of the distribution over $S$ from which to sample $s_{t+1}$. For this, a variety of established methods can be used, from simple rejection sampling in low dimensions (employed in our simulations) to MCMC methods (e.g. slice sampling) in higher dimensions. 
In cetain special cases, sampling from $x_{t}$ may be done very efficiently. For example, if
the losses are affine, the domain~$S$ is a hyperrectangle, and the potential is a generalized Exponential Potential, then $s_{t+1}$ can be obtained by sampling from $n$ independent truncated exponential random variables. 
The main computational challenge is then to compute the integral in $f$. Off-the-shelf numerical integration schemes work well if $n$ is small, but are typically not applicable in higher dimensions. Instead, one has to resort to other methods, such as Monte Carlo methods or sparse grids.

\section{Numerical Results and Comparison With Other Methods}
\label{appsec:NumRes}

In this section, we review some algorithms for online \emph{convex} optimization over subsets of~$\Rbb^n$ that have been proposed in the literature, and compare them with our Dual Averaging method for online optimization on compact metric spaces with uniformly continuous rewards from Section~\ref{sec:OOCUC}. Such algorithms are often formulated in terms of \emph{loss functions}~$\ell_\tau$, but clearly these algorithms apply just as well by setting $\ell_\tau = - u_\tau$, as long as the set $S$ is convex and the rewards are concave and satisfy the additional assumptions made by the algorithms. Table~\ref{table:Bounds} summarizes the regret bounds of each method, with the corresponding assumptions on the feasible set and the loss functions.

The bound on Dual Averaging in Table~\ref{table:Bounds} is obtained by assuming the regularizer to be the $f$-divergence associated to an $\omega$-potential and making an assumption on the asymptotic growth rate of the function $f_\phi$ as follows:

\begin{corollary}
\label{cor:NumRes:BndAsymptotic_fphi}
Suppose that $f_\phi(x) \leq C_\phi\,x^{1+\kappa}$ for some $\kappa >0$ and $C_\phi < \infty$. Suppose further the rewards are $\alpha$-H\"older continuous, i.e. $\chi(r) = C_\alpha\,r^\alpha$, and that $h_\phi$ is uniformly essentially strongly convex with modulus $\gamma(r) = \frac{K}{2} r^2$. Then the learning rate $\eta_t = \eta\,t^{-\beta}$ with $\eta = \frac{1}{M} \bigl( \frac{1+\frac{\kappa}{\alpha}Q}{2+\frac{\kappa}{\alpha}Q}\frac{C_0C_\phi}{c_0^{1+\kappa}\vartheta^{\kappa Q}} \bigr)^{\!1/2}$  and $\beta = \frac{1}{2+\frac{\kappa}{\alpha}Q}$ yields the following bound: 
\begin{align}
\label{eq:NumRes:BndAsymptotic_fphi}
\frac{\Rcal}{t} \leq \bigl( 2M\tilde{C} \vartheta^{-\frac{\kappa Q}{2}} + C_\alpha \vartheta^\alpha \bigr)\, t^{-\frac{1}{2+\frac{\kappa}{\alpha}Q}}
\end{align}
for any $\vartheta < r_0$, where $\tilde{C} = \sqrt{\frac{2+\frac{\kappa}{\alpha}Q}{1+\frac{\kappa}{\alpha}Q} \,\frac{C_0C_\phi}{c_0^{1+\kappa}} }$. 
\end{corollary}


\begin{table}[h]
\begin{tabular}{|l|c|c|c|}
\hline
Algorithm & Assumptions & Parameters & Bound on $\Rcal_t/t$ \\  \hline
\multirow{2}{*}{GP / OGD} & \column{$\ell_t$ convex \\ $\|\nabla \ell_t\|_2 \leq G$} & $\eta_t = \frac{1}{\sqrt t}$ & $ \bigl( \frac{D^2}{2} + G^2 \bigr) t^{-1/2} - \frac{G^2}{2} t^{-1} $ \\[0.5ex] \cline{2-4}
 & \column{$\ell_t$ $H$-strongly convex \\ $\|\nabla \ell_t\|_2 \leq G$} & $\eta_t = \frac{1}{Ht}$ & $ \frac{G^2}{2H} \frac{1+\log t}{t}$ \\ \hline
FTAL / ONS & \column{$\ell_t$ $\alpha$-exp-concave\\ $\|\nabla \ell_t\|_2 \leq G$} & $\beta = \min(\frac{1}{8GD}, \frac{\alpha}{2})$ & $64n(\alpha^{-1} +GD) \frac{1 + \log t}{t} $ \\[1ex] \hline
EWOO & $\ell_t$ $\alpha$-exp-concave & $\eta_t = \alpha$ & $\frac{n}{\alpha} \frac{1+\log (1+t)}{t} $ \\[0.75ex] \hline
DA with $h_\phi$ & \column{$S$ $Q$-regular with $c_0,C_0$ \\ $u_t$ $\alpha$-H\"older-continuous \\ $\|u_t\|_* \leq M$ \\ $f_\phi(x) \leq C x^{1+\epsilon}$ $\forall x \geq 1$} & $\eta_t =\eta\, t^{-\frac{1}{2+\frac{\kappa}{\alpha}Q}}$ & 
$\bigl( 2M\tilde{C} \vartheta^{-\frac{\kappa Q}{2}} + C_\alpha \vartheta^\alpha \bigr)  t^{-\frac{1}{2+n\kappa}}$ \\[0.6ex] \hline
\end{tabular}
\caption{Regret bounds of different online optimization algorithms} 
\label{table:Bounds}
\end{table}%

\subsection{Optimizing Sequences of Convex Functions over Convex Sets}
\label{appsubsec:ConvexLearning}
\cite{Zinkevich:2003aa} formalized the online convex optimization problem, in which the feasible set $S$ and the loss functions are assumed to be convex. He proposed a Greedy Projection method (GP), summarized in Algorithm~\ref{alg:GP}, which we will also refer to as Online Gradient Descent (OGD). 
Theorem~1 in~\citep{Zinkevich:2003aa} shows that when $\|\nabla \ell_t \|$ is uniformly bounded, the regret of GP with learning rates $\eta_t = 1/\sqrt {t}$ grows as $\Ocal(\sqrt t)$. \cite{Hazan:2007aa} show that it is possible to obtain logarithmic regret under additional assumptions on the loss functions. In particular, if the losses are $H$-strongly convex 
then GP with learning rates $\eta_t = \frac{1}{Ht}$ has regret $\Rcal_t \leq \frac{M^2}{2H} \parenth{1 + \log t}$. 
They also propose methods for uniformly exp-concave losses, that is, when there exists $\alpha > 0$ such that $\exp(-\alpha \ell_t)$ is concave for all $t$. These methods, Exponentially Weighted Online Optimization (EWOO) and Follow The Approximate Leader (FTAL), are summarized in Algorithm~\ref{alg:EWOO} and~\ref{alg:ONS} (their Online Newton Step (ONS) algorithm is very similar to FTAL and and therefore omitted). The respective regret bounds are given in Theorems~4 and~7 in~ \citep{Hazan:2007aa} and are summarized in Table~\ref{table:Bounds}. 

%
%

%
\begin{table}[h]
\vspace{-3ex}
\begin{minipage}[t]{.49\textwidth}
\begin{algorithm}[H]
\begin{algorithmic}[1]
\FOR{$t \in \Nbb$}
\STATE Let $\tilde s_{t+1} = s_{t} - \eta_{t+1} \nabla \ell_t(s_t)$
\STATE Update: $x_{t+1} = \delta_{s_{t+1}}$, where
\begin{align*}
s_{t+1} &= \argmin_{s \in S} \|s - \tilde s_{t+1}\|
\end{align*}
\vspace{0.2ex}
\ENDFOR
\end{algorithmic}
\caption{\small Greedy Projection method (GP) a.k.a. Online Gradient Descent (OGD), with input sequence $(\ell_t)$ and learning rates $(\eta_t)$}
\label{alg:GP}
\end{algorithm}
\end{minipage}%
\hspace{.1in}\begin{minipage}[t]{.49\textwidth}
\begin{algorithm}[H]
\begin{algorithmic}[1]
\FOR{$t \in \Nbb$}
\STATE Let $L_{t} = \sum_{\tau  =1}^t \ell_{\tau}$
\STATE Let $\tilde x_{t+1}(s) = \frac{e^{-\alpha L_{t}(s)}}{\int_{S} e^{-\alpha L_{t}(s)}\lambda(ds)} $
\STATE Update: $x_{t+1} = \delta_{s_{t+1}}$, where\\[-1.5ex]
\[
s_{t+1} = \Exp_{s \sim \tilde x_{t+1}} [s]
\]%
\vspace{-1.65ex}
\ENDFOR
\end{algorithmic}
\caption{\small Exponentially Weighted Online Optimization method (EWOO), with input sequence $(\ell_t)$ and learning rate $\alpha$.}
\label{alg:EWOO}
\end{algorithm}
\end{minipage}
\begin{minipage}[t]{.49\textwidth}
\vspace*{-1.5ex}
\begin{algorithm}[H]
\begin{algorithmic}[1]
\FOR{$t \in \Nbb$}
\STATE Let $g_\tau = \nabla \ell_\tau(s_\tau)$
\STATE Let ${A_t = \sum_{\tau = 1}^t  g_\tau(g_\tau)^T}$ and \\$\tilde s_{t+1} = (A_t)^\dagger \bigl(\sum_{\tau = 1}^t g_\tau(g_\tau)^T s_\tau - \frac{1}{\beta} g_\tau\bigr)$,\\  and define $\|s\|_{A_t} = \langle s,A_ts \rangle$.
\STATE Update: $x_{t+1} = \delta_{s_{t+1}}$, where\\[-2.5ex]
\begin{align*}
s_{t+1} &= \argmin_{s \in S} \|s - \tilde s_{t+1}\|_{A_t} 
\end{align*}
\vspace{-.1in}
\ENDFOR
\end{algorithmic}
\caption{\small Follow The Approximate Leader (FTAL) with input sequence $(\ell_t)$ and parameter $\beta$.}
\label{alg:ONS}
\end{algorithm}
\end{minipage}%
\hspace{.1in}\begin{minipage}[t]{.49\textwidth}
\vspace*{-1.5ex}
\begin{algorithm}[H]
\begin{algorithmic}[1]
\FOR{$t \in \Nbb$}
\STATE Let $U_t = \sum_{\tau = 1}^{t} u_\tau$
\STATE Update
\begin{align*}
x_{(t+1)} &= Dh^*(\eta_t U_t) \notag \\
&= \arg \max_{x \in \Xcal} \;\bigl\langle \eta_tU_t, x \bigr\rangle - h(x) 
\end{align*}
\vspace{.235in}
\ENDFOR
\end{algorithmic}
\caption{\small Dual Averaging (DA) with input sequence $(u_t)$, learning rates $(\eta_t)$, and regularizer $h$.}
\label{alg:dual_averaging}
\end{algorithm}
\end{minipage}
\end{table}
%


\begin{example}[Convex Quadratics on a Hypercube]
\label{ex:NumRes:Quad}
%
As a first example, we consider quadratic reward functions of the form $u_t(s) = -\frac{1}{2}(s-\mu_t)^TQ_t(s-\mu_t) - c_t$, where $Q_t$ is  p.d. symmetric,  and $c_t\geq 0$. The domain is $S = \{\|s\|_\infty \leq 0.5\}$ with $D_{\!S} = \sqrt{n}$, and the rewards are generated randomly, $\mathrm{L}$-Lipschitz with $\mathrm{L}= 5$ and uniformly bounded by $\|u_t\|_\infty \leq 3.75$ and $\|u_t\|_4 \leq 1.6$. 
Figure~\ref{fig:NumRes:Quad:full} shows the time-average regrets $\Rcal_t/t$ in dimensions $n=2$ and $n=3$ 
for time horizons of $T=10^4$ and $T= 4\cdot 10^3$, respectively. Displayed are the empirical means over $N=2500$ runs of the algorithm (solid), the associated theoretical bounds\footnote{For easier readability we omitted the bound on FTAL, which in this example is much higher than the others.} (dashed), and the regions between the associated 10\% and 90\% quantiles (shaded).

Not surprisingly, those algorithms that exploit the strong convexity of the problem (OGD, FTAL, EWOO) achieve better asymptotic rates than GP (which requires only convexity) or DA (which makes no convexity assumptions at all). Still, the regret of DA is not significantly higher than that of GP and OGD, and is competitive with FTAL over the simulation horizon. We note that the theoretical regret bounds for both DA instances are much closer to the actual regret of the algorithm. 

\begin{figure}[h]
\centering
\includegraphics[width=\textwidth]{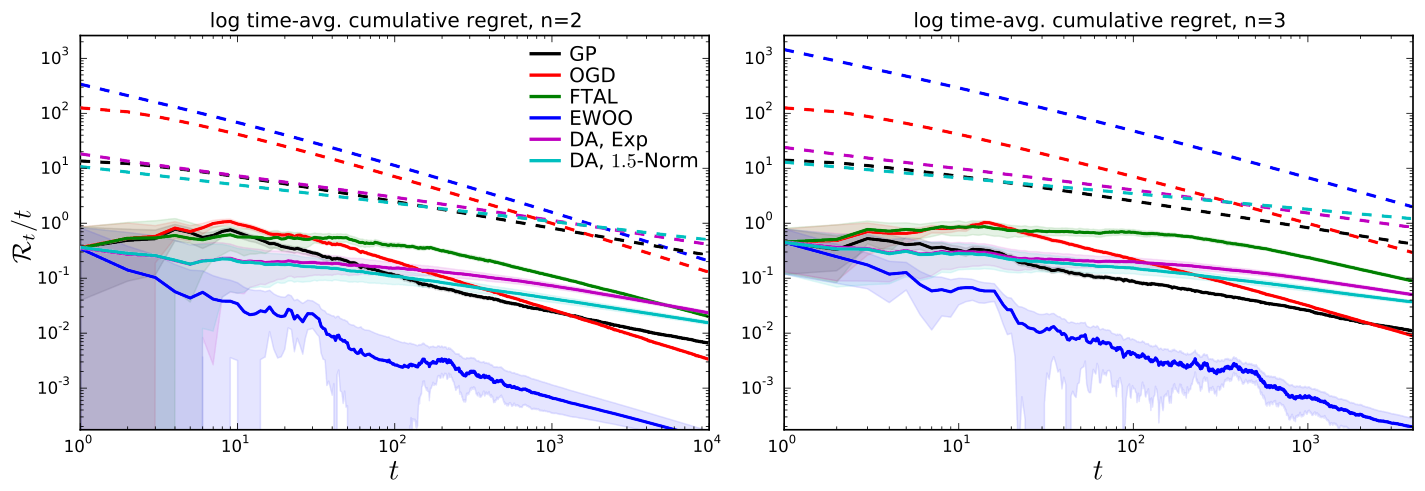}
\vspace{-5ex}
\caption{Time-average regret of different online learning algorithms}
\label{fig:NumRes:Quad:full}
\end{figure}

Table~\ref{table:NumRes:Quad:full} shows the decay rates (which correspond to the slopes in the log-log plots) of empirical means and theoretical bounds in Figure~\ref{fig:NumRes:Quad:full} at the end of the simulation horizon. 
There is a relatively good match between bounds and simulations. Except for FTAL and EWOO, all algorithms exhibit a decay that is faster than that of the associated bound\footnote{For EWOO this discrepancy is likely due to numerical inaccuracies at the very small regrets for large~$t$, while for FTAL the simulation may not have reached the asymptotic regime yet.}. When making this comparison, one must keep in mind that all these bounds are worst-case in nature, and that it is not entirely clear what characterizes a worst-case sequence of reward functions (see Example~\ref{ex:NumRes:Greedy} for a partial remedy).\\[2ex]

  \begin{minipage}[b]{\textwidth}
  \begin{minipage}[b]{0.5\textwidth}
    \centering   
    \small{
\begin{tabular}{l|cc|cc}
 & \multicolumn{2}{c|}{$n=2$} & \multicolumn{2}{c}{$n=3$} \\
Algorithm & simulation & theory & simulation & theory \\ \hline 
GP & -0.564 & -0.497 & -0.515 & -0.495 \\
OGD & -0.920 & -0.900 & -0.892 & -0.888 \\ 
FTAL & -0.780 & -0.900 & -0.705 & -0.888 \\ 
EWOO & -0.809 & -0.900 & -0.676 & -0.888 \\ 
DA, Exp & -0.519 & -0.446 & -0.481 & -0.439 \\
DA, $1.5$-Norm & -0.452 & -0.333 & -0.396 & -0.286
\end{tabular}}
    \vspace{-1ex}
    \captionof{table}{Rates in Figure~\ref{fig:NumRes:Quad:full}}
    \label{table:NumRes:Quad:full}
  \end{minipage}
  \hfill
  \begin{minipage}[b]{0.4\textwidth}
    \centering    
    \small{
\begin{tabular}{l|cc}
Potential & simulation & theory \\ \hline 
ExpPot & -0.557 & -0.446  \\
$1.01$-Norm & -0.546 & -0.495 \\ 
$1.05$-Norm & -0.477 & -0.476  \\ 
$1.5$-Norm & -0.307 & -0.333 \\ 
$1.75$-Norm & -0.279 & -0.286 
\end{tabular}}
    \vspace{-1ex}
    \captionof{table}{Rates in Figure~\ref{fig:NumRes:Greedy:fail}}
    \label{table:RatesGreedyfail}
    \end{minipage}
  \end{minipage}

\end{example}


\begin{example}[Alternating Affine Losses on a Hypercube]
\label{ex:NumRes:Greedy}
In this example we consider a situation in which the greedy algorithm mentioned in Section~\ref{sec:OOCUC} fails\footnote{In fact, any deterministic policy will incur linear regret in a nontrivial adversarial setting.}, and offer a simulation that illustrates the result of Proposition~\ref{prop:OOCUC:LimitThm:strong}. 
We consider a sequence of affine reward functions on $S = \{\|s\|_\infty \leq 0.5\}$ in $\Rbb^2$, alternating in such a way that any maximizer $s_t^\star$ of $U_t$ is in fact a minimizer of $U_{t+1}$. Specifically, we choose $u_t(s) = -\langle a_t, s\rangle - c_t$, where
\begin{align*}
a_0 = [\mathrm{L}/2,0], \quad c_0=\mathrm{L}/4, \quad a_t = [(-1)^t \mathrm{L}, 0], \quad c_t =\mathrm{L}/2
\end{align*} 
for $t\geq 1$. It is easy to see that in this case the greedy algorithm incurs time-average regret $\Rcal_t/t = \mathrm{L} + o(1)$. 

Figure~\ref{fig:NumRes:Greedy:fail} shows regrets for the greedy algorithm and DA with Exponential and different $\rho$-Norm potentials. Besides the obvious failure of the greedy algorithm, we observe that for $p$-Norm potentials  performance decreases as $\rho \searrow 1$, which can be explained by Proposition~\ref{prop:OOCUC:LimitThm:strong}.  
Nevertheless, DA guarantees sublinear regret for any $\rho>1$ (with theoretical asymptotic rate approaching $t^{-1/2}$ as $\rho\rightarrow 1$), though at the cost of much higher constants in the bound as $\rho\approx 1$.  
Table~\ref{table:RatesGreedyfail} shows that empirical and theoretical rates in this instance (which is intuitively hard) are very close, providing further support for the theoretical analysis of DA. Finally, Figure~\ref{fig:NumRes:Greedy:KL} for each potential shows the negative entropy $D_{KL}(x_t || \lambda)$ of $x_t$. 
%
 From this we observe that the minimizers $x^\star[\rho]$ are indeed more and more concentrated around their mode as $\rho \searrow 1$.\\[2ex] 

  \begin{minipage}[b]{\textwidth}
  \begin{minipage}[b]{0.565\textwidth}
  \hspace{-4ex}
    \includegraphics[width=\textwidth]{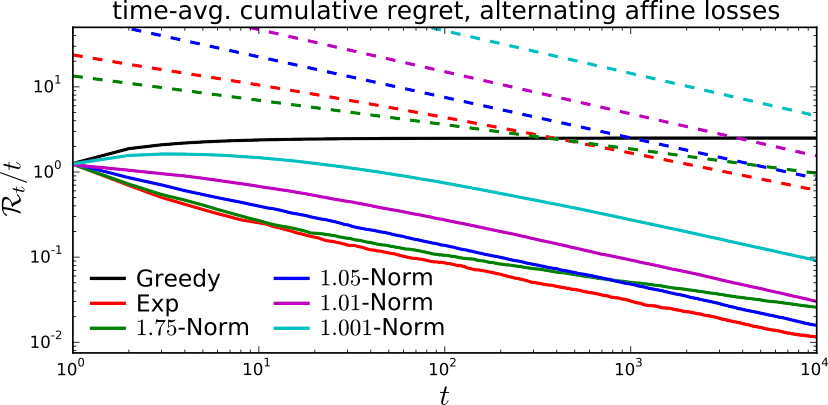}  
    \vspace{-1.75ex}
    \captionof{figure}{Failure of the Greedy Policy}
    \label{fig:NumRes:Greedy:fail}
    \end{minipage}
\hspace{-3ex}
  \begin{minipage}[b]{0.425\textwidth}
    \includegraphics[width=\textwidth]{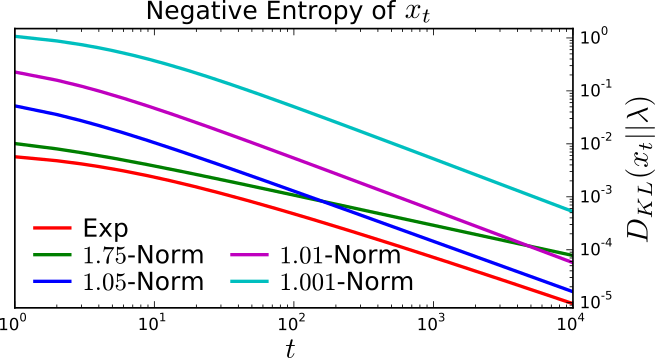}
    \vspace{-1.75ex}
    \captionof{figure}{Negative entropy of $x_t$}
    \label{fig:NumRes:Greedy:KL}
    \end{minipage}
  \end{minipage}

\end{example}

\section{Proofs Omitted in the Main Part}
\label{appsec:OmittedProofs}

\subsubsection*{Proof of Theorem~\ref{thm:RegMin:Prelims:MainTheorem}}

\begin{proof}[Theorem~\ref{thm:RegMin:Prelims:MainTheorem}]
Essential Fr\'echet differentiability, the characterization~\eqref{eq:RegMin:Prelims:MainTheorem:FrechetGradient} of the Fr\'echet gradient in~\ref{thm:RegMin:Prelims:MainTheorem:FrechetGradient} and~\ref{thm:RegMin:Prelims:MainTheorem:StrongSmooth} follow  from Theorem~\ref{thm:RegMin:Prelims:MainTheorem:FrechetDiffability}, Lemma~\ref{lem:RegMin:Prelims:PWEquiv}, and the definition of uniform essential strong convexity. 
To prove~\eqref{eq:RegMin:Prelims:MainTheorem:ModContGrad}, let $\xi_1, \xi_2 \in X^*$ and let $x_i = Df^*(\xi_i) = \argmax_{x\in X} \langle \xi_i, x \rangle - f(x)$. Then, by first-order optimality, $\langle z - \xi_i, x - x_i\rangle \geq 0, \; \forall\,z\in \partial f(x_i), \forall\, x\in X$. 
In particular, 
\begin{align*}
\langle z_1 - \xi_1, x_2 - x_1\rangle &\geq 0 \\
\langle z_2 - \xi_2, x_1 - x_2\rangle &\geq 0 
\end{align*}
for all $z_i \in \partial f(x_i)$, $i=1,2$. Summing these inequalities we find that 
\begin{align*}
\langle \xi_1 - \xi_2, x_1 - x_2 \rangle \geq \langle z_1 - z_2, x_1 - x_2 \rangle
\end{align*}
By uniform strong convexity, we further have that $f(x) \geq f(x_i) + \langle x-x_i, z_i \rangle + \gamma(\|x-x_i\|)$ for all $x\in X$. In particular, 
\begin{align*}
f(x_1) &\geq f(x_2) + \langle x_1-x_2, z_2 \rangle + \gamma(\|x_1-x_2\|) \\
f(x_2) &\geq f(x_1) + \langle x_2-x_1, z_1 \rangle + \gamma(\|x_2-x_1\|) 
\end{align*}
and summing these inequalities yields
\begin{align*}
\langle z_1 - z_2, x_1 - x_2 \rangle \geq 2 \gamma(\|x_1-x_2\|)
\end{align*}
On the other hand, $\langle \xi_1 - \xi_2, x_1 - x_2 \rangle \leq \| \xi_1 - \xi_2 \|_* \|x_1 - x_2 \|$ by definition of the dual norm, so 
\begin{align*}
\tilde\gamma(\|x_1-x_2\|) \leq  \frac{1}{2} \| \xi_1 - \xi_2 \|_*
\end{align*}
using the definition of $\tilde\gamma$. If $\tilde\gamma$ is strictly increasing it admits a (strictly increasing) inverse~$\tilde\gamma^{-1}$. Applying~$\tilde\gamma^{-1}$ to both sides then yields~\eqref{eq:RegMin:Prelims:MainTheorem:ModContGrad}. 
\end{proof}

\subsubsection*{Proof of Theorem~\ref{thm:RegMin:DiscTime:DTBounds}}

\begin{proof}[Theorem~\ref{thm:RegMin:DiscTime:DTBounds}]
We consider the continuous-time reward and learning rate processes $u^c$ and $\eta^c$ given by 
$u^c_t := u_{\lceil t \rceil}$ and $\eta^c(t) := \eta_{\lfloor t \rfloor \vee 1}$, respectively, where $\lceil r \rceil := \inf \{ n\in \Zbb \,\suchthat\, n\geq r \}$ and $\lfloor r \rfloor =  \sup \{ n\in \Zbb \,\suchthat\, n\leq r \}$ for all $r\in \Rbb$ and $a\vee b = \min(a,b)$. 
In doing so we follow the ideas of the analysis of~\cite{Kwon:2014aa} (our problem is, however, different as our reward vectors are infinite-dimensional). With this
\begin{align*}
x_k = Dh^*\biggl(\eta_{k-1} \sum_{j=1}^{k-1}u_j\biggr) =  Dh^*\biggl(\eta^c(k-1) \int_0^{k-1} u^c_\tau \,d\tau \biggr) =  x^c_{k-1}
\end{align*}
and thus, for $j\geq 1$ and $t\in (j-1, j)$, we have
\begin{align}
|\langle u^c_t \mid x^c_t \rangle - \langle u_j , x_j\rangle| = |\langle  u_j , x^c_t - x^c_{j-1} \rangle| \leq \|u_j\|_* \| x^c_t - x^c_{j-1}\|
\end{align}
by definition of the dual norm. 
Therefore
\begin{align}
|\langle u^c_t , x^c_t \rangle - \langle u_j , x_j\rangle| \leq \|u_j\|_* \|Dh^*(y^c_t) - Dh^*(y^c_{j-1})\| \leq \|u_j\|_* \, \tilde\gamma^{-1} \bigl( \|y^c_t - y^c_{j-1}\|_* /2 \bigr)
\end{align}
where the second inequality follows from Theorem~\ref{thm:RegMin:Prelims:MainTheorem}. From the definition of $y_t^c$, we have 
\begin{align*}
\|y^c_t - y^c_{j-1}\|_* = \biggl\| \eta^c(j-1) \int_{j-1}^t u^c_\tau \,d\tau \biggr\|_* \leq \eta_{j-1} \|u_j\|_* (t-j+1)
\end{align*}
and therefore 
\begin{align*}
\biggl| \int_0^k \langle u^c_\tau, x^c_\tau\rangle d\tau - \sum_{j=1}^k \langle u_j, x_j \rangle \biggr| &\leq \sum_{j=1}^k \int_{j-1}^j |\langle u^c_\tau,x^c_\tau \rangle - \langle u_j,x_j\rangle | \,d\tau \\
&\leq \sum_{j=1}^k \|u_j\|_* \int_{j-1}^j \tilde\gamma^{-1} \biggl(\frac{\eta_{j-1} (t-j+1)}{2}\|u_j\|_* \biggr) \,d\tau \\
&\leq \sum_{j=1}^k \|u_j\|_*  \tilde\gamma^{-1} \Bigl(\frac{\eta_{j-1}}{2}\|u_j\|_* \Bigr) 
\end{align*}
where the last equality follows since $\tilde\gamma^{-1}$ is non-decreasing (a consequence of $\gamma$ being sublinear).
Finally, we note that
\begin{align*}
\sum_{j=1}^k \langle u_j, x \rangle - \sum_{j=1}^k \langle u_j, x_j \rangle &= \int_{0}^k \langle u^c_\tau, x \rangle\, d\tau - \sum_{j=1}^k \langle u_j, x_j \rangle \\
&\leq  \biggl| \int_{0}^k \langle u^c_\tau, x \rangle\, d\tau - \int_0^k \langle u^c_\tau, x^c_\tau\rangle d\tau \biggr|  + \biggl| \int_0^k \langle u^c_\tau, x^c_\tau\rangle d\tau - \sum_{j=1}^k \langle u_j, x_j \rangle \biggr| 
\end{align*}
The bound~\eqref{eq:RegMin:DiscTime:DTBounds:Bound} then follows from Theorem~\ref{thm:DAContTime:MainBound} and the above. 
 \end{proof}

\subsubsection*{Proof of Corollary~\ref{cor:RegMin:DiscTime:DTBounds:superlinear}}

\begin{proof}[Corollary~\ref{cor:RegMin:DiscTime:DTBounds:superlinear}]
It is easy to show that $\tilde\gamma^{-1} \bigl( \frac{\eta_{j-1} }{2}\|u_j\|_* \bigr) \leq (2C)^{-1/\kappa}\, \eta_{j-1}^{1/\kappa}\, \|u_j\|_*^{1/\kappa}$. If $\|u_j\|_* \leq M$ for all~$j$, then
$ R_t(x) \leq \frac{h(x)-\underline{h}}{\eta^c_t} + (2C)^{-1/\kappa} M^{1+1/\kappa} \sum_{\tau=1}^t  \eta_{\tau-1}^{1/\kappa}$.
In particular, if $\eta_t = \eta\, t^{-\beta}$, then~\eqref{eq:RegMin:DiscTime:DTBounds:superlinear:etaspecific} follows from the bound 
$\sum_{\tau=1}^t  {(j-1)}^{-\beta/\kappa} \leq \int_0^t v^{-\beta/\kappa} dv = \frac{\kappa}{\kappa-\beta} t^{1-\frac{\beta}{\kappa}}$
\end{proof}

\subsubsection*{Proof of Theorem~\ref{thm:OOCUC:DiscT}}

\begin{proof}[Theorem~\ref{thm:OOCUC:DiscT}]
The space $X = L^p(S)$ is uniformly convex~\citep{Clarkson:1936aa}, and thus reflexive~\citep{Milman:1938aa}. 
Its dual is $X^* = L^q(S, \mu)$ for $q=\frac{p}{p-1}$ and $\langle x, \xi \rangle = \int_S x(s)\xi(s) \,\mu(ds)$ for $x\in X$ and $\xi\in X^*$. 
Fix~$t<\infty$. Then for any $s\in S$ and all $x\in \Bcal(s,\vartheta_t)$ 
\begin{align*} 
\langle U_t, x \rangle &= \int_{B(s,\vartheta_t)} \!\! U_t(s') x(s') \, d\mu(s') = \int_{B(s,\vartheta_t)} \sum_{\tau=1}^t u_\tau(s') x(s') \, d\mu(s') \,d\tau\,\\
&\geq \sum_{\tau=1}^t \int_{B(s,\vartheta_t)} \bigl( u_\tau(s) - \chi(\vartheta_t) \bigr) x(s') \, d\mu(s') \,d\tau\,
= U_t(s)  - t\, \chi(\vartheta_t)
\end{align*}
and therefore
\begin{align*}
\Rcal_t &= \sup_{s\in S} \, U_t(s)  - \sum_{\tau=1}^t \langle u_\tau, x_\tau \rangle \\[-1ex] 
&\leq \adjustlimits{\sup}_{s\in S}{\inf}_{x\in \Bcal(s,\vartheta_t)} \langle U_t, x \rangle +  t\, \chi(\vartheta_t) - \sum_{\tau=1}^t \langle u_\tau, x_\tau \rangle  \\
&= \adjustlimits{\sup}_{s\in S}{\inf}_{x\in \Bcal(s,\vartheta_t)} R_t(x)  + t \,\chi(\vartheta_t)
\end{align*}
and thus~\eqref{eq:OOCUC:Bound} follows from~\eqref{eq:DAContTime:MainBound} in Theorem~\ref{thm:DAContTime:MainBound}.
\end{proof}

\subsubsection*{Proof of Proposition~\ref{prop:OOCUC:EquivOfSuprema}}

Denote by $\Onebf_A$ the indicator function of the set $A$, i.e. $\Onebf_A(s) = 1$ if $s\in A$ and $\Onebf_A(s) = 0$ if $s\not\in A$. In this proof we will make use of the following Lemma:

\begin{lemma}
\label{lem:pf:thm:LearnCont:Repeated:EquivOfSuprema}
Let $(S,d)$ be a compact metric space and let  $\mu$ be an $r_0$-locally $Q$-regular measure on $S$. For $p\geq 1$ let $\Xcal^p := \{ x \in L^p(S, \mu) \,\suchthat\, x\geq 0\,\text{ a.s.}, \, \|x\|_1 = 1 \}$. Suppose further that 
 $f:S \rightarrow \Rbb$ is continuous. Then
\begin{align}
\label{eq:LearnCont:TPCS:EquivOfSuprema}
 \sup_{s \in S} f(s) =  \sup_{x\in \Pcal} \int_{S} f(s)\,dx(s) = \sup_{x\in \Xcal^p} \int_{S} f(s)\,dx(s),\quad\quad \forall\, p \in [0,\infty]
\end{align}
\end{lemma}

\begin{proof}[Lemma~\ref{lem:pf:thm:LearnCont:Repeated:EquivOfSuprema}]
The first equality follows directly by observing that Borel measures measures include measures with finite support. Clearly $\sup_{x\in \Pcal} \int_{S} f(s)\,dx(s) \geq \sup_{x\in \Xcal^p} \int_{S} f(s)\,dx(s)$ since $\Xcal^p \subset \Pcal$ for all $p\in [1,\infty]$. Since $L^p \subset L^q$ for all $q\geq p$ it suffices to show the reverse inequality holds for $p=\infty$. Since $S$ is compact and $f$ is continuous, there exists a maximizer $s^\star$ of $f$ on $S$. Let $\epsilon > 0$. By continuity, there exists $\delta>0$ such that $|f(s)-f(s')|\leq \epsilon$ whenever $d(s,s')<\delta$. Moreover, by local $Q$-regularity of $\mu$ we have that $\mu(B(s^\star,\delta)) >0$. Now let $x(s) = \frac{1}{\mu(B(s^\star,\delta))} \Onebf_{B(s^\star,\delta)}(s)$. Clearly $x\in \Xcal^\infty$, and
\begin{align*}
\int_{S_i} \!\!f(s)\,dx(s) &= \frac{1}{\mu(B(s^\star,\delta))} \int_{B(s^\star,\delta)} \!\!\ f(s)\,d\lambda(s) 
\geq \frac{1}{\mu(B(s^\star,\delta))} \int_{B(s^\star,\delta)} \!\!\!\! (f(s^\star) - \epsilon)\,d\lambda(s) 
= f(s^*) - \epsilon
\end{align*}
Now let $\epsilon \searrow 0$.
\end{proof}

\begin{proof}[Proposition~\ref{prop:OOCUC:EquivOfSuprema}]
Recall that 
\begin{align*}
R_t(x) &= \sum_{\tau=1}^t \langle u_\tau, x \rangle - \sum_{\tau=1}^t \langle u_\tau, x_\tau \rangle = \int_S U_t(s)\,dx(s) - \sum_{\tau=1}^t \langle u_\tau, x_\tau \rangle
\end{align*}
Clearly $U_t$ is continuous (in fact, with modulus of continuity $t\,\chi(r)$) on $S$ for any $t<\infty$. The equivalence of the suprema then follows from a direct application of Lemma~\ref{lem:pf:thm:LearnCont:Repeated:EquivOfSuprema}. 
\end{proof}

\subsubsection*{Proof of Proposition~\ref{prop:OOCUC:fDiv}}

\begin{proof}[Proposition~\ref{prop:OOCUC:fDiv}]
By convexity of $f$, we have that $h(x) = h_{\phi}(x) \geq f_\phi \bigl(\int_S \frac{dx}{d\mu} d\mu(s)\bigr) = f_\phi(1) = 0$ for all $x\in \Xcal$, and thus $\underline{h} =0$. Furthermore, choosing $x$ as the uniform Radon-Nikodym density w.r.t. $\mu$ on $B(s,\vartheta_t)$, i.e., 
\begin{align*}
x(s')= \frac{\Onebf_{B(s,\vartheta_t)}(s')}{\mu(B(s,\vartheta_t))}
\end{align*}
we have that 
\begin{align*}
h(x) &= \int_{S_i} f_\phi(x(s'))\,\mu(ds') = \int_{B(s,\vartheta_t)} f_\phi\biggl(\frac{1}{\mu(B(s,\vartheta_t))} \biggr) \,\mu(ds') \\
&\leq \min \bigl(C_0 (\vartheta_t)^Q\!,\, \mu(S)\bigr) \, f_\phi\biggl(\frac{1}{\mu(B(s,\vartheta_t))} \biggr)
\end{align*}
where we used the assumption of $r_0$-local $Q$-regularity and the fact that $\vartheta_t \leq r_0$. It is easy to see that $f_\phi$ is increasing on $[1,\infty)$. Indeed, $f_\phi'(x) = \phi^{-1}(x)$, and $\phi^{-1}(x)$ is increasing by assumption with $\phi^{-1}(1) \geq 0$. 
%
 %
Moreover, since $\mu(S) = 1$ by assumption, we have that $\mu(B(s,\vartheta_t)) \leq 1$ for any $s$, so 
\begin{align*}
h(x) \leq  \min \bigl(C_0 (\vartheta_t)^Q\!,\, \mu(S)\bigr) \, f_\phi\bigl( c_0^{-1} (\vartheta_t)^{-Q} \bigr)
\end{align*}
Plugging this into the general bound~\eqref{eq:OOCUC:DiscT:Bound} of Theorem~\ref{thm:OOCUC:DiscT} yields~\eqref{eq:OOCUC:fDiv:Bound}. 
\end{proof}

%

\subsubsection*{Proof of Corollary~\ref{cor:OOCUC:omegaPot:ExpGenHedge:Bound}}

\begin{proof}[Corollary~\ref{cor:OOCUC:omegaPot:ExpGenHedge:Bound}]
Plugging $\tilde{\gamma}(r) = 2r$,  $f_\phi(x) = x\log x$  and $\chi(r) = C_\alpha r^\alpha$ into~\eqref{eq:OOCUC:fDiv:Bound} we find that 
\begin{align*}
\frac{\Rcal_t}{t}  \leq \frac{C_0}{c_0\, t\, \eta_t}  \log \bigl( c_0^{-1} \vartheta_t^{-Q} \bigr) + C_\alpha \vartheta_t^\alpha + \frac{M^2}{t} \sum_{\tau=1}^t \eta_{\tau-1}
\end{align*}
Letting $\eta_t = \eta \sqrt{\log t}\;  t^{-\beta} $ we have that
\begin{align*}
\sum_{\tau=1}^t \eta_{\tau-1} \leq \eta \sqrt{\log t} \sum_{\tau=1}^t t^{-\beta} \leq \eta \sqrt{\log t} \sum_{\tau=1}^t \int_{\tau-1}^\tau z^{-\beta}dz = \eta \sqrt{\log t} \int_0^t z^{-\beta} dz = \frac{\eta \sqrt{\log t}}{1-\beta} t^{1-\beta}
\end{align*}
 and therefore
\begin{align*}
\frac{\Rcal_t}{t}  \leq \frac{C_0}{c_0\eta } \frac{t^{\beta-1}}{\sqrt{\log t}} \log \bigl( c_0^{-1} \vartheta_t^{-Q} \bigr) +C_\alpha \vartheta_t^\alpha + \frac{\eta M^2}{1-\beta} \sqrt{\log t} \;t^{-\beta}
\end{align*}
Choosing $\beta = 1/2$ and $\vartheta_t = \vartheta^\frac{1}{\alpha} (\log t)^\frac{1}{2\alpha} t^{-\frac{\beta}{\alpha}}$ this becomes, after dropping a $1/\log t$ term,
\begin{align*}
\frac{\Rcal_t}{t} &\leq \biggl( \frac{C_0}{c_0\,\eta } \biggl( \log(c_0^{-1}\vartheta^{-Q/\alpha}) + \frac{Q}{2\alpha} \biggr) + C_\alpha \vartheta + 2 \eta M^2  \biggr) \sqrt{\frac{\log t}{t}}
\end{align*}
as $\vartheta_t < r_0$  since $\sqrt{\log t/ t} < \vartheta^{-1}r_0^\alpha$. 
Then choosing $\eta= \frac{1}{M} \sqrt{\frac{C_0Q}{2c_0} \log(c_0^{-1}\vartheta^{-Q/\alpha}) + \frac{Q}{2\alpha}}$ gives
\begin{align*}
\frac{\Rcal_t}{t} &\leq \biggl( 2M \sqrt{ \frac{2C_0}{c_0} \Bigl( \log(c_0^{-1}\vartheta^{-Q/\alpha}) + \frac{Q}{2\alpha} \Bigr) } + C_\alpha \vartheta \biggr) \sqrt{\frac{\log t}{t}}
\end{align*}
\end{proof}

\subsubsection*{Proof of Theorem~\ref{thm:OOCUC:LowerBound}}

\begin{proof}[Theorem~\ref{thm:OOCUC:LowerBound}]
Since $S$ is compact there exist $s^a, s^b \in S$ such that $d(s^a,s^b) = D_{\!S}$. Let $x^a = \delta_{s^a}$ and $x^b = \delta_{s^b}$, where~$\delta_s$ denotes the Dirac measure on~$S$ at~$s$. Let $w:\Rbb\rightarrow \Rbb$ be any function with modulus of continuity $\chi$ such that $\|w(d(\,\cdot\,,s^b))\|_q \leq M$. Define $v :S \rightarrow \Rbb$ by $v(s) = w(d(s,s^b))$. Using the triangle inequality it is easy to see that $v$ also has modulus of continuity $\chi$. Now observe that
\begin{align*}
\langle v, x^a - x^b\rangle = v(s^a) - v(s^b) = w(d(s^a,s^b)) = w(D_S)
\end{align*}
Let $V_1,\dotsc,V_2$ a sequence of i.i.d. Rademacher random variables, i.e. $\Pbb(V_i = +1) = \Pbb(V_i=-1) = \frac{1}{2}$, and consider the (random) sequence of reward vectors $(u_\tau)_{\tau=1}^t$ with $u_t = V_t v$. By Proposition~\ref{prop:OOCUC:EquivOfSuprema} we have that $\Rcal_t = \sup_{x\in \Pcal} R_t(x)$, and thus
\begin{align*}
\Ebb[\Rcal_t] &= \Ebb\biggl[ \; \sup_{x\in \Pcal} \sum_{\tau=1}^t \langle u_\tau, x \rangle - \sum_{\tau=1}^t \langle u_\tau, x_\tau \rangle \biggr] \geq  \Ebb\biggl[ \; \max_{x\in \{x^a, x^b\}} \sum_{\tau=1}^t \langle u_\tau, x \rangle \biggr] - \Ebb\biggl[ \; \sum_{\tau=1}^t \langle u_\tau, x_\tau \rangle \biggr] \\
&=  \Ebb\biggl[ \; \max_{x\in \{x^a, x^b\}} \sum_{\tau=1}^t V_\tau \langle v, x \rangle \biggr] - \Ebb\biggl[ \; \sum_{\tau=1}^t V_\tau \langle v, x_\tau \rangle \biggr] 
\end{align*}
Observe that the second expectation is zero for any sequence of $(x_\tau)_{\tau=1}^t$ with $x_\tau$ measurable with respect to $\sigma(V_1,\dotsc,V_{\tau-1})$, i.e. any online algorithm. Noting that $\max(a,b) = \frac{1}{2}(a+b) + \frac{1}{2}|a-b|$ we thus have that
\begin{align*}
\Ebb[\Rcal_t] &\geq  \frac{1}{2}\, \Ebb\biggl[ \; \sum_{\tau=1}^t V_\tau  \langle v, x^a + x^b \rangle \biggr] + \frac{1}{2}\, \Ebb\biggl[ \;\biggl|  \sum_{\tau=1}^t V_\tau \langle v, x^a - x^b \rangle \biggr| \biggr] \\
&= \frac{w(D_{\!S})}{2} \Ebb\biggl[ \;\biggl|  \sum_{\tau=1}^t V_\tau \biggr| \biggr] \geq  \frac{w(D_{\!S})}{2\sqrt{2}}  \sqrt{t}
\end{align*}
where the last step follows from an application of Khintchine's inequality~\citep{Haagerup:1981aa}.
\end{proof}

\subsubsection*{Proof of Proposition~\ref{prop:OOCUC:LowerBoundHoelder}}

\begin{lemma}
\label{lem:OOCUC:HoelderLemma}
Let $C\in \Rbb$ and $0<\beta\leq 1$. The function $v: [0,\infty)$ given by $v(r) = C r^\beta$ is H\"older continuous with modulus of continuity $\chi(r) = |C|^{\beta} r^\beta$.
\end{lemma}

\begin{proof}[Lemma~\ref{lem:OOCUC:HoelderLemma}]
Noting that $|x+y|^\beta \leq |x|^\beta + |y|^\beta$ for any $x,y \in \Rbb$ we find with $x= Cr_1 - Cr_2$ and $y=Cr_2$ for any $r_1, r_2\geq 0$ that
$|C|r_1^\beta - |C|r_2^\beta \leq |C|^\beta |r_1-r_2|^\beta$. Exchanging the roles of $r_1$ and $r_2$ then yields $\bigl| Cr_1^\beta - Cr_2^\beta \bigr| \leq |C|^\beta |r_1-r_2|^\beta$.
\end{proof}

\begin{proof}[Proposition~\ref{prop:OOCUC:LowerBoundHoelder}]
With $s^a, s^b$ as in the proof of Theorem~\ref{thm:OOCUC:LowerBound}, choose
\begin{align*}
w(r) = \min \bigl( C_\alpha^{1/\alpha}\,,\, M \|d(\,\cdot\,, s^b)^\alpha \|_q^{-1} \bigr) \, r^\alpha
\end{align*}
Then clearly $\|w(d(\,\cdot\,,s^b))\|_q \leq M$ by construction. Moreover, $w$ has modulus of continuity $\tilde{\chi}(r) \leq C_\alpha r^\alpha$ by Lemma~\ref{lem:OOCUC:HoelderLemma}. The result  follows from observing that $\|d(\,\cdot\,,s^b)^\alpha\|_q \leq \|D_{\!S}^\alpha\|_q = D_{\!S}^\alpha$. 
\end{proof}

\subsubsection*{Proof of Proposition~\ref{prop:OOCUC:LimitThm:strong}}

\begin{proof}[Proposition~\ref{prop:OOCUC:LimitThm:strong}]
Fix $t<\infty$ and let $\delta>0$. Consider $x\in \Xcal$ with $\epsilon :=\int_{(B_{\delta}^*)^c} x(s) \mu(ds) > 0$. %
and define the function $\kappa: \Rbb^+ \rightarrow \Rbb^+$ as $\kappa(u) = \sup_{s\in (B_u^*)^c} U_t(s)$. Clearly, $\kappa$ is decreasing, $\kappa(u) < U^*$ for $u>0$ by definition of $S^*$, and continuous (by continuity of $U_t$). We then have that $U_t(s) < U^* - \kappa(d(s,S^*))$ for all $s\in S$. 
Let $0<\delta' < \chi^{-1}(\frac{\kappa(\delta)}{2\,t})$ such that $\mu(B_{\delta'}^*) > 0$. Such a $\delta'$ always exists by $Q$-regularity of $\mu$. Consider 
\begin{align*}
\tilde{x}(s) = x(s) {\Onebf}_{B_\delta^*}(s) + \frac{\epsilon}{\mu(B_{\delta'}^*)} {\Onebf}_{B_{\delta'}^*}(s)
\end{align*}
Clearly, $\tilde{x} \in \Xcal$.   
Furthermore, 
\begin{align*}
(*) &:=\int_S \eta_{t} U_t(v)\tilde{x}(v) \mu(dv)  - h_{i}(\tilde{x}) - \int_S \eta_{k}U_k(v)x(v) \mu(dv)  + h_i(x) \\[-1ex]
&= \frac{\epsilon}{\mu(B_{\delta'}^*)} \int_{B_{\delta'}^*} \!\!\! \eta_k U_k(v) \mu(dv) -  \int_{(B_{\delta}^*)^c} \!\!\!\!\! \eta_t U_k(v) x(v) \mu(dv) - (h_i(\tilde{x}) - h_i(x)) \\[-0.5ex]
&\geq  \epsilon\,\eta_k (U^* - t \chi(\delta')) - \epsilon \eta_t (U^* - \kappa(\delta)) - (h_i(\tilde{x}) - h_i(x)) \\
&\geq \epsilon\, \eta_t( \kappa(\delta) - t\chi(\delta') ) - (h_i(\tilde{x}) - h_i(x)) \\
&> \frac{\epsilon\,\eta_t\, \kappa(\delta)}{2\,t} - (h_i(\tilde{x}) - h_i(x))
\end{align*}%
Now $h_i(\tilde{x}) - h_i(x) \rightarrow 0$ as $i \rightarrow \infty$ by consistency of $(h_i)_{i\geq 0}$. Hence there exists $j<\infty$ such that $(*) > 0$ and thus $x \neq x_j^*$ for all $i \geq j$. Since $\epsilon$ was arbitrary, this shows that $\int_{(B_{\delta}^*)^c} x_i^*(s) \,\mu(ds) \rightarrow 0$ as $i \rightarrow \infty$. 
\end{proof}

\subsubsection*{Proof of Corollary~\ref{cor:OOCUC:LimitThm:weak}}

\begin{proof}[Corollary~\ref{cor:OOCUC:LimitThm:weak}]
Let $f: S \rightarrow \Rbb$ be continuous and bounded, say $|f(s)|\leq M$ for all $s\in S$. Let $\epsilon > 0$. Since $S$ is compact, $f$ is uniformly continuous, i.e. $\exists\,\delta>0$ such that $|f(s) - f(s^*)| < \epsilon/2$ for~all $s \in B_{\delta}^*$. By Corollary~\ref{prop:OOCUC:LimitThm:strong} there exists $j<\infty$ such that $x_i^*((B_{\delta}^*)^c) <\frac{\epsilon}{4M}$ for all $i>j$. Hence 
\begin{align*}
\int_S |f(s)-f(s^*)| x_i^*(s) \lambda(ds) < \epsilon/2 \int_{B_{\delta}^*} \! x_i^*(s) \mu(ds) +  2M \int_{(B_{\delta}^*)^c} \! x_i^*(s) \mu(ds) < \epsilon
\end{align*}
for all $i > j$.
\end{proof}

\subsubsection*{Proof of Proposition~\ref{prop:LearnCont:Repeated:limsupV}}

\begin{proof}[Proposition~\ref{prop:LearnCont:Repeated:limsupV}]
This proof uses similar arguments as Theorem~7.2 in~\cite{Cesa-Bianchi:2006aa}, with modifications to accommodate our more general setting of functions on metric spaces. 

Since player~1 has sublinear (realized) regret, by~\eqref{eq:LearnCont:Repeated:HannanConistency} it suffices to show that 
\begin{align*}
\sup_{s^1\in S_1} \frac{1}{t} \sum_{\tau=1}^t u(s^1, s^2_\tau) \geq V.
\end{align*}
Now clearly $\sup_{s^1\in S_1} f(s^1) = \sup_{x^1\in \Pcal_1} \int_{S_i} f(s) \,dx^1(s)$ for any $f$ measurable, thus we may equivalently show that 
 $\sup_{x^1\in \Pcal_1} \frac{1}{t} \sum_{\tau=1}^t \int_{S_1} u(s^1, s^2_\tau)\, dx^1(s^1) \geq V$. Observe that, for all $x^1\in \Pcal_1$, 
\begin{align*}
\frac{1}{t} \sum_{\tau=1}^t \int_{S_1} u(s^1, s^2_\tau)\, dx^1(s^1)  &=  \int_{S_1} \frac{1}{t} \sum_{\tau=1}^t  u(s^1, s^2_\tau)\, dx^1(s^1) \\
&= \int_{S_1} \frac{1}{t} \sum_{\tau=1}^t  \Bigl( \int_{S_2} u(s^1, s)\, d\delta_{s^2_\tau}(s) \Bigr) \, dx^1(s^1) \\
&=\bar{u}(x^1, \hat{x}^2_t)
\end{align*}
where $\hat{x}^2_t(B) := \frac{1}{t} \sum_{\tau=1}^t \Onebf_{B}(s^2_\tau)$ for any Borel set $B\subset S_2$. Since $\hat{x}^2_t \in \Pcal_2$ we thus have that 
\begin{align*}
\sup_{x^1\in \Pcal_1} \bar u(x^1, \hat x^2_t)  \geq \adjustlimits{\inf}_{x^2\in \Pcal_2}{\sup}_{x^1\in \Pcal_1} \bar u(x^1, x^2) = V
\end{align*}
\end{proof}

%
%

\subsubsection*{Proof of Corollary~\ref{cor:LearnCont:Repeated:Value}}

\begin{proof}[Corollary~\ref{cor:LearnCont:Repeated:Value}]
Using the fact that the payoff of player~2 is the negative of player~1, we have from Theorem~\ref{prop:LearnCont:Repeated:limsupV} and the fact that the game has a value that
\begin{align*}
\liminf_{t\rightarrow \infty} \frac{1}{t} \sum_{\tau=1}^t -u(s^1_\tau, s^2_\tau) \geq -V
\end{align*}
and thus 
\begin{align*}
\limsup_{t\rightarrow \infty} \frac{1}{t} \sum_{\tau=1}^t u(s^1_\tau, s^2_\tau) \leq V
\end{align*}
Combining this with~\eqref{eq:LearnCont:Repeated:limsupV} proves~\eqref{eq:LearnCont:Repeated:Value}.
\end{proof}

\subsubsection*{Proof of Theorem~\ref{thm:LearnCont:Repeated:ConvOfEmpPlay}}

In the proof of the theorem we will use the following Lemma:

\begin{lemma}
\label{lem:pf:thm:LearnCont:Repeated:ConvOfEmpPlay}
The functions  $g_1(x^2) := \sup_{x^1\in \Pcal_1} \bar{u}(x^1,x^2)$ and $g_2(x^1) := \inf_{x^2\in \Pcal_2} \bar{u}(x^1,x^2)$ are continuous with respect to the weak topology. 
\end{lemma}

\begin{proof}[Lemma~\ref{lem:pf:thm:LearnCont:Repeated:ConvOfEmpPlay}]
It suffices to show that $g_1^{-1}((-\infty, a))$ and $g^{-1}((b,\infty))$ are open, since the sets of the form $(-\infty,a)$ and $(b,\infty)$ form a subbase for the topology of~$\Rbb$. Observe first that $u$ is continuous. Indeed, by Assumption~\ref{ass:LearnCont:HannanStrategies:UnifContUnif}, we have  for any $s, t \in S_1\times S_2$ that 
 \begin{align*}
 | u(s^1,s^2) - u(t^1,t^2) | &\leq   | u(s^1,s^2) - u(s^1,t^2) | + |u(s^1,t^2) - u(t^1,t^2) | \\
 &\leq \chi^2(d_2(s^2,t^2)) + \chi^1(d_1(s^1,t^1))
 \end{align*}
and so for any $\epsilon >0$ there exists $\delta>0$ such that $|u(s^1,s^2) - u(t^1,t^2) | < \epsilon$ whenever $(d_1\times d_2)(s,t) <\delta$. 
Since $u$ is continuous on the compact set $S_1 \times S_2$ it is bounded, i.e. there exists $M< \infty$ such that $|u(s^1,s^2)| \leq M$ for all $s\in S$. This implies that $\bar{u}(x^1,x^2)$ is $2M$-Lipschitz w.r.t the L\'evy-Prokhorov metric on $\Pcal_1\times \Pcal_2$, hence in particular (jointly) continuous w.r.t. the weak (product) topology. Let $\pi_{2}: \Pcal_1\times\Pcal_2 \rightarrow \Pcal_2$ denote the canonical projection onto $\Pcal_2$, which by definition of the product topology is continuous. Together with the continuity of $\bar{u}$ this implies that $g_1^{-1}((b,\infty)) = \pi_{2} \circ \bar{u}^{-1}((b,\infty))$ is open. Furthermore, note that $\bar{u}(x^1,x^2) < a, \,\forall\,x^1 \in \Pcal_1$ whenever $g_1(x) < a$, and hence for any $x^2\in \Pcal_2$, the set $(x^1,x^2) \in g_1^{-1}((-\infty,a))$ is open. That is, there exists an open cover of $\Pcal_1 \times \{x^2\}$. Now $\Pcal_1$ is compact in the weak topology, which means we can find a finite subcover $\{U_{x^2}^j\}_{j=1}^{n_{x^2}}$ such that $ \bigcap_{j=1}^{n_{x^2}} U_{x^2}^j \supset \Pcal_1 \times \{x_2\}$. Taking the union over all $x^2 \in g^{-1}((-\infty, a))$ we have that $g^{-1}((-\infty,a)) = \bigcup_{x^2\in g^{-1}((-\infty,a))} \bigcap_{j=1}^{n_{x^2}} U_{x^2}^j$, which is an open set. This shows that $g_1$ is continuous. The argument for showing continuity of $g_2$ is essentially the same. 
\end{proof}

\begin{proof}[Theorem~\ref{thm:LearnCont:Repeated:ConvOfEmpPlay}]
Note that both $\Pcal_i$ are metrizable and compact in the weak topology (as each $S_i$ is compact), and hence $\Pcal_1\times\Pcal_2$ by Tychonoff's theorem. Therefore it suffices to show that with probability $1$, the weak limit of any weakly converging subsequence of $(\hat{x}_t)_{t=0}^\infty$ is a Nash equilibrium. Let $(\hat{x}^1_\theta, \hat{x}^2_\theta)_{\theta=1}^\infty$ be such weakly convergent subsequence, and $(z^1,z^2) \in \Pcal_1\times\Pcal_2$ its weak limit. We will show that whenever a given realization of plays $(s^1_t)$, $(s^2_t)$ has sublinear regret for both players, $(z^1,z^2)$ is a Nash Equilibrium, i.e.,
\begin{align}
\label{eq:LearnCont:Repeated:ConvOfEmpPlay:pf:Saddle}
 \sup_{x^1\in \Pcal_1} \bar{u}(x^1,z^2) =  V = \inf_{x^2\in \Pcal_2} \bar{u}(z^1,x^2).
\end{align} 
Let $g_1(x_2) := \sup_{x^1\in \Pcal_1} \bar{u}(x^1,x^2)$ and $g_2(x^1) := \inf_{x^2\in \Pcal_2} \bar{u}(x^1,x^2)$, which by Lemma~\ref{lem:pf:thm:LearnCont:Repeated:ConvOfEmpPlay} are continuous w.r.t. the weak topology. Hence, using that $\hat{x}^i_\theta \rightharpoonup z^i$ for $i=1,2$, \eqref{eq:LearnCont:Repeated:ConvOfEmpPlay:pf:Saddle} is equivalent to
\begin{subequations}
\label{eq:LearnCont:Repeated:ConvOfEmpPlay}
\begin{align}
\adjustlimits{\lim}_{\theta\rightarrow \infty}{\sup}_{x^1\in \Pcal_1} \bar{u}(x^1,\hat{x}^2_\theta) &= V,
\label{eq:LearnCont:Repeated:ConvOfEmpPlay:1} \\
\adjustlimits{\lim}_{\theta\rightarrow \infty}{\inf}_{x^2\in \Pcal_2} \bar{u}(\hat{x}^1_\theta,x^2) &=V.
\label{eq:LearnCont:Repeated:ConvOfEmpPlay:2}
\end{align}
\end{subequations}
We first show~\eqref{eq:LearnCont:Repeated:ConvOfEmpPlay:1}. By assumption, the game has value~$V$, i.e. it holds that $\inf_{x^2\in \Pcal_2} \sup_{x^1\in \Pcal_1} \bar{u}(x^1,x^2) = V$
and thus, in particular, that
\begin{align}
\label{eq:LearnCont:Repeated:ConvOfEmpPlay:pf:liminf}
\adjustlimits{\liminf}_{\theta\rightarrow \infty}{\sup}_{x^1\in \Pcal_1} \bar{u}(x^1,\hat{x}^2_\theta) \geq V.
\end{align}
Now, suppose that for a realization $(s^1_\tau), (s^2_\tau)$, the regret of the second player is sublinear, i.e.
\begin{align*}
\limsup_{t\rightarrow \infty} \frac{1}{t} \biggl(\; \sup_{x^1\in \Pcal_1} \sum_{\tau=1}^t \int_{S_1} u(s^1,s^2_\tau)\,dx^1(s^1) - \sum_{\tau=1}^n u(s^1_\tau, s^2_\tau) \biggr) \leq 0.
\end{align*}
Then by Corollary~\ref{cor:LearnCont:Repeated:Value}, $\lim_{t\rightarrow \infty} \frac{1}{t} \sum_{\tau=1}^t u(s^1_\tau, s^2_\tau) = V$, and we have
\begin{align*}
V &\geq \limsup_{t\rightarrow \infty}  \sup_{x_1\in \Pcal_1} \frac{1}{t} \sum_{\tau=1}^t   \int_{S_1}  u(s^1,s^2_\tau)\,dx^1(s^1) \\
&=  \limsup_{t\rightarrow \infty} \sup_{x_1\in \Pcal_1}  \int_{S_1}   \frac{1}{t} \sum_{\tau=1}^t u(s_1,s^2_\tau)\,dx^1(s^1) \\
&=  \limsup_{t\rightarrow \infty} \sup_{x_1\in \Pcal_1} \bar{u}(x^1,\hat{x}^2_t) \\
&\geq  \limsup_{\theta\rightarrow \infty} \sup_{x_1\in \Pcal_1} \bar{u}(x^1,\hat{x}^2_\theta).
\end{align*}
Combining the last inequality with~\eqref{eq:LearnCont:Repeated:ConvOfEmpPlay:pf:liminf} proves~\eqref{eq:LearnCont:Repeated:ConvOfEmpPlay:1}. The argument for~\eqref{eq:LearnCont:Repeated:ConvOfEmpPlay:2} is essentially the same, modulo some sign changes.

This proves that for any realization with sublinear regret for both players, all weak limit points of the sequence $(\hat x_t^1, \hat x^2_t)$ lie in the set of Nash equilibria. But by definition of Hannan consistency, this happens with probability $1$.

%

\end{proof}

\subsubsection*{Proof of Theorem~\ref{thm:LearnCont:HannanStrategies:MainConsistency}}

\begin{proof}[Theorem~\ref{thm:LearnCont:HannanStrategies:MainConsistency}]
To start, note that for any $p>1$ the space $\Xcal$ as a closed subset of $L^p(S,\mu)$ is a complete metric space, hence Polish and thus there exists a Borel isomorphism between $\Xcal$ and the Lebesgue measure on the unit interval. Consequently, to randomize its plays according to a sequence of probability measures in $\Xcal$, it suffices that player~$i$ has access to a sequence of i.i.d. random variables drawn from the uniform distribution on $[0,1]$. Denote this sequence by $Z^i = (Z^i_1,Z^i_2,\dotsc)$.

The key observation is that if player~$-i$ plays a non-oblivious strategy, then the partial rewards will not be some a priori fixed sequence of reward functions, but will depend on the history of play. Indeed, since $\tilde{u}^i_t(\,\cdot\,) = \sum_{\tau=1}^t u_i(\,\cdot\,,s^{-i}_\tau)$ and since $s^{-i}_\tau$ is itself some function of past plays $s^i_1, \dotsc,s^i_{\tau-1}$, the partial reward functions $\tilde{u}^i_t$ are measurable w.r.t. the $\sigma$ field generated by $(Z^i_1,\dotsc,Z^i_{t})$. Note that this implicitly assumes that any randomization performed by player~$-i$ is independent of that of player~$i$. 
%
Let $\Ebb^i_t [X] := \Ebb \bigl[ X \mid Z^{i}_1,\dotsc, Z^{i}_{t-1} \bigr]$ denote the conditional expectation of $X$ given the past plays of player~$i$. 
Then
\begin{align}
\sum_{\tau=1}^t u^i(s^i, s^{-i}_\tau) - \sum_{\tau=1}^t \Ebb^i_\tau \bigl[u^i(s^i_\tau, s^{-i}_\tau) \bigr]
&\leq \sum_{\tau=1}^t\; \sup_{s^{-i}_\tau} \; \Ebb^i_\tau \bigl[ u_i(s^i, s^{-i}_\tau) -  u_i(s^i_\tau, s^{-i}_\tau) \bigr] \nonumber \\
&= \sum_{\tau=1}^t \; \sup_{\tilde{u}^i_\tau} \; \Ebb^i_\tau \bigl[ \tilde{u}^i_\tau(s^i) -  \tilde{u}^i_\tau(s^i_\tau) \bigr] \nonumber \\
&=  \sup_{\tilde{u}^i_1,\dotsc,\tilde{u}^i_t} \; \sum_{\tau=1}^t \; \Ebb^i_\tau \bigl[ \tilde{u}^i_\tau(s^i) -  \tilde{u}^i_\tau(s^i_\tau) \bigr] \label{pf:thm:LearnCont:HannanStrategies:MainConsistency:1}
\end{align}
where the last step uses the fact that $s^i_\tau \sim x^i_\tau := Dh_i^*\bigl( \eta_{\tau-1} \textstyle \sum_{\theta =1}^{t-1} \tilde{u}^i_\theta \bigr)$, which depends on the sequence $\{s^i_\theta\}_{\theta=1}^{\tau-1}$ only through the sequence  $\{\tilde{u}^i_\theta\}_{\theta=1}^{\tau-1}$ of observed partial loss functions. 

From Proposition~\ref{prop:OOCUC:EquivOfSuprema} we have that 
\begin{align}
\label{pf:thm:LearnCont:HannanStrategies:MainConsistency:2}
\Rcal_t = \sup_{s^i\in S_i}\; \sup_{\tilde{u}^i_1,\dotsc, \tilde{u}^i_t}  \; \sum_{\tau=1}^t \tilde{u}^i_\tau(s^i) - \sum_{\tau=1}^t \langle \tilde{u}^\tau_i, x^i_\tau \rangle = \sup_{s^i\in S_i}  \sup_{\tilde{u}^i_1,\dotsc,\tilde{u}^i_t} \; \sum_{\tau=1}^t \; \Ebb^i_\tau \bigl[ \tilde{u}^i_\tau(s^i) -  \tilde{u}^i_\tau(s^i_\tau) \bigr] 
\end{align}

Now let $W^i_\tau = \tilde{u}^i_\tau(s^i_\tau) - \langle \tilde{u}^i_\tau, x^i_\tau \rangle$ and observe that $W^i_\tau$ is a martingale. Indeed,
\begin{align*}
\Ebb [W^i_\tau \mid W^i_\tau, \dotsc, W_i^{\tau-1} ] = \Ebb [W^i_\tau \mid Z_i^\tau, \dotsc, Z_i^{\tau-1} ] = 0\quad\quad \text{a.s.}
\end{align*}
Moreover, since by assumption $u_i$ is continuous on the compact set $S_1\times S_2$, we have that $u_i$ is bounded and therefore $|W^i_{\tau} - W^i_{\tau-1}| \leq M$ for some $M<\infty$. Noting that $W^i_\tau = 0 $ it follows from the Azuma-Hoeffding inequality that,  for every $\epsilon>0$,  $\Pbb(W^i_\tau \leq \epsilon) \geq 1 - \exp(-\frac{\epsilon^2}{2\tau M^2})$ and thus 
\begin{align*}
\Pbb\bigl( \textstyle \sum_{\tau=1}^t W^i_\tau \leq M\sqrt{2t\log (t/\epsilon)}  \bigr)  \geq 1-\epsilon \quad\quad \forall\,\epsilon >0
\end{align*}
Now $\sum_{\tau=1}^t W^i_\tau = \sum_{\tau=1}^t \tilde{u}^i_\tau(s^i_\tau) - \sum_{\tau=1}^t \langle \tilde{u}^i_\tau, x^i_\tau \rangle$, and hence, using~\eqref{pf:thm:LearnCont:HannanStrategies:MainConsistency:2} and~\eqref{pf:thm:LearnCont:HannanStrategies:MainConsistency:1}, we have for all $t<\infty$ that 
\begin{align*}
\sup_{s^i\in S_i} \frac{1}{t} \biggl( \; \sum_{\tau=1}^t u_i(s^i, s^{-i}_\tau) - \sum_{\tau=1}^t u_i(s^i_\tau, s^{-i}_\tau) \biggr)
\leq \frac{\Rcal_t}{t} + M \sqrt{\frac{2 \log (t/\epsilon)}{t}}
\end{align*}
Now $\Rcal_t/t \rightarrow 0$ by assumption, and $\sqrt{\frac{ \log (t/\epsilon)}{t}} \rightarrow 0$ for any $\epsilon >0$, which proves Hannan consistency. 
\end{proof}

\end{document}

%% file: latex_commands.tex
\newcommand{\Nbb}{\mathbb{N}}

\newcommand{\Pbb}{\mathbb{P}}
\newcommand{\Rbb}{\mathbb{R}}
\newcommand{\Zbb}{\mathbb{Z}}

\newcommand{\Bcal}{\mathcal{B}}

\newcommand{\Gcal}{\mathcal{G}}

\newcommand{\Ncal}{\mathcal{N}}
\newcommand{\Ocal}{\mathcal{O}}
\newcommand{\Pcal}{\mathcal{P}}

\newcommand{\Rcal}{\mathcal{R}}

\newcommand{\Ucal}{\mathcal{U}}

\newcommand{\Xcal}{\mathcal{X}}

\newcommand{\Zcal}{\mathcal{Z}}

\newcommand{\RN}[1]{%
  \textup{\uppercase\expandafter{\romannumeral#1}}%
}

\DeclareMathOperator*{\argmax}{arg\,max}
\DeclareMathOperator*{\argmin}{arg\,min}

\DeclareMathOperator*{\dom}{dom}

\DeclareMathOperator*{\Onebf}{{\mathbf{1}}}

\DeclareMathOperator*{\interior}{int}

\newcommand{\Exp}{\mathbb{E}}
\newcommand\Ebb{\mathbb E}

\newcommand\ind[1]{\iota_{#1}}
\newcommand{\suchthat}{{\, : \,}}

\newcommand \parenth[1]{\left( #1 \right)}

\usepackage{listings} 
\usepackage{framed}